\documentclass{article}

\usepackage{microtype}
\usepackage{graphicx}
\usepackage{subfigure}
\usepackage{booktabs} % for professional tables
\usepackage{enumitem}
\usepackage[pagebackref=true]{hyperref}

\renewcommand*{\backref}[1]{}
\renewcommand*{\backrefalt}[4]{{%
    [\ifcase #1 Not cited.%
          \or p. #2%
          \else pp. #2%
    \fi%
    ]}}

\usepackage{shortcuts}

\usepackage[accepted]{icml2024}

\usepackage{amsmath}
\usepackage{amssymb}
\usepackage{mathtools}
\usepackage{amsthm}

\usepackage[capitalize,noabbrev]{cleveref}

\theoremstyle{plain}
\newtheorem{theorem}{Theorem}[section]
\newtheorem{proposition}[theorem]{Proposition}
\newtheorem{lemma}[theorem]{Lemma}
\newtheorem{corollary}[theorem]{Corollary}

\theoremstyle{definition}
\newtheorem{definition}{Definition}[section]

\newtheorem{assumption}{Assumption}[section]
\theoremstyle{remark}
\newtheorem{remark}{Remark}[section]
\crefname{assumption}{Assumption}{Assumptions}

\icmltitlerunning{Switching the Loss Reduces the Cost in Batch Reinforcement Learning}

\begin{document}

\twocolumn[
\icmltitle{Switching the Loss Reduces the Cost in Batch Reinforcement Learning}



\icmlsetsymbol{equal}{$\dagger$}

\begin{icmlauthorlist}
\icmlauthor{Alex Ayoub}{alberta}
\icmlauthor{Kaiwen Wang}{cornell}
\icmlauthor{Vincent Liu}{alberta}
\icmlauthor{Samuel Robertson}{alberta}
\icmlauthor{James McInerney}{netflix}
\icmlauthor{Dawen Liang}{netflix}
\icmlauthor{Nathan Kallus}{netflix,cornell}
\icmlauthor{Csaba Szepesvári}{alberta}
\end{icmlauthorlist}

\icmlaffiliation{alberta}{University of Alberta}
\icmlaffiliation{cornell}{Cornell University}
\icmlaffiliation{netflix}{Netflix, Inc.}

\icmlcorrespondingauthor{Alex Ayoub}{aayoub14k@outlook.com}

\icmlkeywords{Machine Learning, ICML}

\vskip 0.3in
]



\printAffiliationsAndNotice{}  % leave blank if no need to mention equal contribution

\begin{abstract}
    We propose training fitted Q-iteration with log-loss ({FQI}-\textsc{log}\xspace) for batch reinforcement learning (RL). We show that the number of samples needed to learn a near-optimal policy with {FQI}-\textsc{log}\xspace scales with the accumulated cost of the optimal policy, which is zero in problems where acting optimally achieves the goal and incurs no cost. In doing so, we provide a general framework for proving \textit{small-cost} bounds, i.e. bounds that scale with the optimal achievable cost, in batch RL. Moreover, we empirically verify that \fqilog uses fewer samples than FQI trained with squared loss on problems where the optimal policy reliably achieves the goal. 
\end{abstract}

%Intro
\section{Introduction}
In batch reinforcement learning (RL), also known as offline RL, 
the goal is to learn a good policy from a fixed dataset.
%we often want learners that can achieve a goal from a fixed dataset using as few samples as possible.
A standard approach in this setting is fitted Q-iteration (FQI) \cite{JMLR:v6:ernst05a}, which
iteratively obtains a sequence of value functions by fitting the next value function to a target that is based
on the data and the previously obtained value function.
In this work we propose a simple and principled improvement to FQI, using \textit{log-loss} (\fqilog), 
which is applicable when the returns along trajectories lie in a bounded interval.
We prove that
the number of samples the new method requires to learn a near-optimal policy scales with the cost of the optimal policy, leading to a so-called \emph{small-cost} bound, the RL analogue of a small-loss bound in supervised learning. 
Such bounds predict improved sample efficiency in goal oriented RL tasks where the goal is reliably achievable and the cost is set up to penalize failure in achieving the goal; a prediction we validate empirically.
We highlight that \fqilog is the first \emph{computationally efficient} batch RL algorithm to achieve a small-cost bound, provided that a regression oracle is available; a condition that can be met, e.g.,  when 
logit models are used in \fqilog.
%logit models are used to estimate the value functions. 
%, i.e. \fqilog only requires minimizing a convex loss function, 

Most earlier works in batch RL  focused on algorithms that achieve the optimal worse-case dependence on the number of samples required to learn a near-optimal policy \cite{munos2003error,antos-nips-2007,pmlr-v97-chen19e,pmlr-v139-xie21d}. 
The only work prior to ours that is known to \emph{adaptively} improve sample efficiency when facing a task with near-zero optimal cost is due to
\citet{wang2023benefits}, who obtain small-cost bounds for finite-horizon batch RL problems
but using the so-called ``distributional RL'' approach. In this approach, one solves 
the regression problems arising when estimating the distribution of a policy's accumlated cost. 
The inspiration for our work comes from this work, combined with the insights of 
\citet{abeille2021instance,foster2021efficient} who showed that, in the simpler bandit problems, log-loss alone is sufficient for obtaining small-cost bounds.
%They attribute this success to 
% under the distributional Bellman completeness assumption, which is stronger than the one made in \cref{asp:completeness}.
%are neither \textit{adaptive} nor \textit{data-dependent}, as they demonstrate the same convergence behavior towards an optimal policy regardless of the specific problem instance they are applied to. 
%Algorithms that enjoy small-cost bounds, such as \fqilog, converge faster in scenarios where the optimal policy reliably incurs a small cost, such as goal-oriented control problems where the goal can be reliable achieved, a property that can be highly advantageous.

%In order to achieve faster and more efficient learning in complex environments, it is essential to design adaptive algorithms for sequential decision-making problems. Such algorithms are able to move beyond the worst-case and leverage favorable properties of specific problem instances.
% Designing adaptive algorithms for sequential decision-making problems that move beyond the worst-case and leverage the favorable properties of a given problem instance is essential for achieving faster and more 
%efficient learning in complex environments.

Why log-loss gives a small-cost advantage is subtle.
When used with an unrestricted model class (think: finite state-action space, ``tabular learning''), both log- and squared losses achieve small cost bounds
because they share the same minimizers, which predict the empirical mean of responses for all inputs.
However, when the model class is restricted (the only practical case for large problems), log-loss and squared loss will trade off errors at the various inputs differently. 
Specifically, with log-loss, the penalty for predicting a value far from an observed mean
 increases rapidly as the observed mean gets close to the boundary of its range,
 an effect that is absent with squared loss.
Consequently, log-loss will favor predictors that fit well to those observed values that are near the boundary of the range,
%The closer an observed value is to the boundary, the more attention it will get by log-loss. 
making the learning process disregard large variance observed values, stabilizing learning, a property not shared when a squared loss is used.
As a result, as we show, under suitable additional technical assumptions, log-loss based FQI achieves small-cost bounds, a property that is not shared when squared loss is used with FQI.
%
%accumulated costs for each state-action pair are not equally noisy. Whereas training with squared loss amounts to assuming the noise in the accumulated costs is \textit{homoscedastic} at each state-action pair, training with log-loss amounts to assuming that the noise in the accumulated costs is \textit{heteroscedastic}.
%Since in RL typically the accumulated costs at certain state-action transitions are noisier than others, learning to emphasize less noisy transitions leads to more sample efficient learning. \todos{It's mildly wierd to say the cost at the transition is noisy and then hit them with a deterministic cost function $c$, maybe we should just say that transitions at certain state-action pairs are noisier?}

The main contributions of this work can be summarized as: {\em (i)}
%\begin{enumerate}
    % \item We demonstrate how to derive small-cost bounds for batch RL under the appropriate assumptions.
%    \item 
    We propose training FQI with log-loss and prove it enjoys a small-cost bound. This is the first efficient batch RL algorithm that achieves a small-cost bound.
%    \item 
When showing this result, we make two technical contributions that may be of independent interest:
{\em (ii)} We show that the Bellman optimality operator is a contraction with respect to the Hellinger distance, a result;
%\end{enumerate}
{\em (iii)}
%We also highlight an intermediate result that may be of independent interest. 
We present a general result that decomposes the suboptimality gap of the value of a greedy policy induced by some value function, $f$, into the product of a small-cost term and the pointwise triangular deviation of $f$ from $q^\star$, the value function of the optimal policy. 
%Intro

%Preliminaries
\section{Preliminaries}\label{sec:prelim}
In this section we review some definitions and concepts of Markov Decision Processes (MDPs) and define the notation used. Readers unfamiliar with the basic theory of MDPs are recommended to consult the book of \citet{bertsekas2019reinforcement}, or that of \citet{Szepesvari2010c}. All results mentioned in this section can be found in these works.
% In this section we review the basic definitions for Markov Decision Processes (MDPs), as well as some extra notation. We assume that readers are familiar with these concepts and the basic theory of MDPs. Readers unfamiliar with this literature are advised to consult the book of \citet{bertsekas2019reinforcement} or that of \citet{Szepesvari2010c}.

An infinite-horizon discounted Markov Decision Process (MDP) is given by a tuple \todos{we already introduced the acronym just above, is this desired?}
$M = (\mcS,\mcA,P,c,\gamma)$, where
$\mcS$ is the  state space, $\mcA$ is the  action space, 
$P$ is a transition function, $c:\cS\times \cA \rightarrow \mbR$ 
is a cost function and $\gamma\in [0,1)$ is a discount factor.
%$P: \mcS\times \mcA \rightarrow \mcM_1(\mcS)$ is the transition function,
%$c: \mcS \times \mcA \rightarrow [0,\infty)$ is the cost function, 
%and $\gamma \in [0,1)$ is the discount factor.
We only consider MDPs with finite action spaces.
Furthermore, for simplicity, we assume that the state space is finite.
Among other things,%
\footnote{We assume that the state space $\mcS$ is finite solely for exposition. This allows us to simplify the presentation of our analysis and focus on the most salient details of the proof, avoiding the cumbersome measure-theoretic notation required to reason about infinite sets.}
 this allows writing the transition function as $P: \mcS\times \mcA \rightarrow \mcM_1(\mcS)$,
where $\mcM_1(\cS)$ denotes the set of probability distributions over $\cS$. 
Since the set $\cS$ is finite, any element of $\cM_1(\cS)$ can and will be identified with its 
probability mass function. The notation $\cM_1(\mcX)$ will be used in the same way to denote the set of probability distributions over an arbitrary finite set $\mcX$ and we will perform the same identification.

A (general) policy $\pi = (\pi_h)_{h=1}^\infty$ is a sequence of functions $\pi_h : (\mcS \times \mcA)^{h-1} \times \mcS \rightarrow \mcM_1(\mcA)$. 
Fixing the start state $s$, a policy $\pi$ induces a distribution $\mathbb{P}_{\pi,s}$ 
over trajectories $S_1,A_1,C_1,S_2,A_2,C_2,\dots,$
where $S_1 = s$, $A_1 \sim \pi_1(S_1)$, $C_1 = c(S_1,A_1)$, $S_2 \sim P(S_1,A_1)$,
$A_2 \sim \pi_2(S_1,A_1,S_2)$: the policy is used to govern the selection of actions, while the transition dynamics of the MDP governs the evolution of the states in response to the chosen actions.
We will also need 
 \textit{stationary Markov} policies, where the choice of the action in any timestep $h$ only depends on the last state visited. Thus, such policies can be identified with a map $\pi : \mcS \rightarrow \mcM_1(\mcA)$, an identification which we will employ in what follows.
 % When such a policy is used, the distribution of action $A_t$ is given by $\pi(S_t)$: $A_t\sim \pi(S_t)$. 

The expected total discounted cost over trajectories starting in $s$ quantifies the policy's performance when initialized in state $s$. We collect these expectations in the \textit{state-value function} of $\pi$, $v^\pi : \mcS\to\mbR$, which is defined by
 $v^\pi(s) = \mathbb{E}_{\pi,s}[\sum_{h=1}^\infty \gamma^{h-1}C_h]$, where
 $\mathbb{E}_{\pi,s}$  is the expectation operator corresponding to $\mathbb{P}_{\pi,s}$.
%, the distribution over trajectories induced by following policy $\pi$ from initial state $s$.
\emph{For convenience,  
throughout this paper we assume that \emph{costs are normalized} so that 
the sum of discounted costs along any trajectory satisfies }
\begin{align}\label{eq:costnorm}
0\le \sum_{h=1}^{\infty} \gamma^{h-1}C_h \leq 1\,.
\end{align}
\todoc{We may need to highlight this better.}By appropriately rescaling the costs, this constraint can always be satisfied (when the state space is not finite, one needs that the above infinite sums are uniformly bounded to be able to do this). \todoc{comment after the result on how rescaling and shifting impacts the bounds.}
%for all possible sequences $\{(S_h,A_h,C_h)\}_{h \in \mathbb{N}}$ of trajectories. 

The \textit{action-value function} of $\pi$, $q^\pi: \mcS \times \mcA \to \mbR$, is defined as
\begin{align*}
q^\pi(s,a)= c(s,a) + \gamma \sum_{s'\in\mcS}P(s'|s,a)v^\pi(s')\,, 
\end{align*}
where $(s,a)\in \cS \times \cA$ and, by abusing notation, we use $P(s'|s,a)$ to denote the probability of landing in state $s'$ when action $a$ is taken in state $s$.
%\todos{Should we replace the sum by $\ip{P(s,a)}{v^\pi}$? Cs: If $\ip{\cdot,\cdot}$ is defined before this point and if the definition allows it to be used with $P(s,a)$ (type compatibility).}
For a stationary Markov policy $\pi$, the state- and action-value functions are related by the identity 
  $v^\pi(s) = \sum_{a \in \mcA} \pi(a|s) q^\pi(s,a)$. Here, and in what follows, abusing notation once again, $\pi(a|s)$ denotes the probability that is assigned by $\pi(s)$ to action $a\in \cA$.

%The state value function is\todoa{Fix these definitions for mathematicians} and the state-action value function is
%$$q^\pi(s,a) = \mathbb{E}[\sum_{h=1}^\infty \gamma^{h-1}C_h | S_1 = s, A_1 = a, A_{2:\infty}\sim\pi]\,.$$

\newcommand{\mcZZ}{\mcS \times \mcA}
%In the case of MDPs with finite state and action spaces, it is well known that there exists an \emph{optimal policy}. 
The \emph{optimal policy} is defined as any policy $\pi^\star$ that satisfies $v^{\pi^\star}(s) = \min_\pi v^\pi(s)$ simultaneously for all $s \in \mcS$. Such a policy exists in our case.
We define the \emph{optimal state-value function} as $v^\star = v^{\pi^\star}$ and the \emph{optimal action-value function} as $q^\star = q^{\pi^\star}$.
% due to the finite state and action spaces, both $v^\star$ and $q^\star$ are well-defined.
Any policy that is \emph{greedy with respect to $q^\star$}, i.e., at state $s$ selects only actions $a$ that minimize $q^\star(s,\cdot)$, is guaranteed to be optimal. 
Furthermore,
the optimal action-value function $q^\star$ satisfies the \emph{Bellman optimality equation} $q^\star = \mcT q^\star$, where \hbox{$\mcT : \mbR^{\mcZZ} \rightarrow \mbR^{\mcZZ}$} is the 
\emph{Bellman optimality operator}, defined via
\begin{align}\label{eqn:bellman-operator}
    (\mcT f)(s,a) &= c(s,a) + \gamma \sum_{s' \in S} P(s'|s,a) \min_{a'\in \mcA} f(s',a'),
\end{align}
for $f:\mcZZ \to \mbR$ and $(s,a)\in \mcZZ$.

% As usual,  $\pi^\star$ will denote a policy that minimizes $v^\pi$ simultaneously over all states $s$. As is well known, such \emph{optimal policies} are guaranteed to exist for our case of finite MDPs.
% We let $v^\star = v^{\pi^\star}$ and $q^\star = q^{\pi^\star}$ .
% Both of these are well defined and they are known as the \emph{optimal state-}, and
% \emph{optimal action-value functions}, respectively.
% Furthermore, any greedy policy with respect to $q^*$ is an optimal policy. Furthermore,
% the optimal action-value function $q^\star$ satisfies the Bellman equation $q^\star = \mcT q^\star$ where $\mcT : \mbR^{\mcS \times \mcA} \rightarrow \mbR^{\mcS \times \mcA}$ is the Bellman optimality \todoc{I changed this from Bellman update operator to Bellman optimality operator, which is the standard terminology across the classic literature on the topic. And is a better name anyways.} operator, defined as follows: For $f:\mcZZ \to \mbR$ and $(s,a)\in \mcZZ$, let\todoc{please avoid $\forall$ etc.}
% \begin{align}\label{eqn:bellman-operator}
%     (\mcT f)(s,a) &= c(s,a) + \gamma \sum_{s' \in S} P(s'|s,a) \min_{a\in \mcA} f(s',a)\,.
% \end{align}

We will find it helpful to use a shorthand for the function $s \mapsto \min_{a\in \mcA} f(s,a)$ appearing above. For $f: \mcZZ \to \mbR$, define $f^\wedge: \mcS \to \mbR$ by \todoc{I expanded this as this notation is somewhat nonstandard. Hence, it is worth giving it space}
\begin{align*}
f^\wedge(s) = \min_{a\in \mcA} f(s,a) \,, \qquad s\in \mcS\,.
\end{align*}
With this notation, we have that for $(s,a) \in \mcZZ$
\begin{align*}
    (\mcT f)(s,a) &= c(s,a) + \gamma \sum_{s' \in S} P(s'|s,a) f^\wedge(s')\,.
     %\\ & \qquad \qquad \qquad \qquad \qquad (s,a) \in \mcZZ\,.
\end{align*}
Finally, we use $\pi_f$ to denote a greedy policy induced by $f$: $\pi_f(s) = \argmin_{a \in \mcA} f(s,a)$. When there are multiple such policies,
we choose one in an arbitrary (systematic) manner to make $\pi_f$ well-defined. \todos{Subscripting $\pi$ does conflict with the nonstationary policies being indexed by timesteps. Cs: Different types; not a problem.}
\todos{The old definition was $\pi_f(s) = \argmin_{a \in \mcA} f(s,a)$, I changed it to make it more precise but maybe it is less clear now. Cs: I don't see the difference.}
% \begin{definition}[Goal-Oriented MDP]\label{defn:goal-mdp}
%     An MDP $M=(S,A,P,C,\gamma,\eta_1,S_g)$ is said to be a goal-oriented MDP if there exists a set of states $S_g$ and a terminating state $s_\tau \notin S_g$ such that for all $s \in S_g, s' \notin S_g$ it holds that $C(s,\cdot) = 1$, $C(s',\cdot)=0$ and $P(s_\tau|s,\cdot) = P(s_\tau|s_\tau,\cdot) = 1$.
% \end{definition}
% \begin{remark}
%     For the analysis of FQI, \cref{defn:goal-mdp} will not be used.
% \end{remark}

\section{Problem Definition: Batch RL} 
In this work, we consider batch reinforcement learning problems, where a learner is given 
a sequence of data points $D_n = \{(S_i,A_i,C_i,S_i')\}_{i=1}^n$ such that $S_i,S_i'\in \cS$, $A_i\in \cA$. and $C_i\in \R$. %[0,\infty)$.
Importantly, the learner has no access to $P$ or $c$.
The learner's goal is to find a policy $\pi$ such that executing the policy $\pi$ from an initial state $S_1$ drawn randomly from some distribution $\eta_1\in \cM_1(\cS)$ results in an expected total discounted cost exceeding the smallest possible such value with as little as possible.
Formally, the learner is evaluated by comparing the value $\bar v^{\pi} = \ip{\eta_1}{v^\pi} = \sum_{s \in \mcS} \eta_1(s)v^\pi(s)$, where $\ip{\cdot}{\cdot}$ denotes the standard inner product, of the policy $\pi$ that it returns with the smallest possible such value
$\bar{v}^\star$, which, thanks to our earlier definitions is easily seen to satisfy 
$\bar v^\star = \ip{\eta_1}{v^\star}$. 
Our algorithm does not need to know $\eta_1$.
%The learner could be given $\eta_1$; but our algorithm does not require this.
%\todoc{the learner does not need to know $\eta_1$; the concentrability coefficient uses $\eta_1$. but the theorem could be stated for any $\eta_1$ with $C(\eta_1)$ denoting the contrentrability of $\mu$ w.r.t. $\eta_1$. This would be much better than the current presentation}

The data is assumed to satisfy the following assumption:
\begin{assumption}[Data/MDP properties]\label{asp:data}
    We have that $D_n=\{(S_i,A_i,C_i,S_{i}')\}_{i=1}^n$ is a sequence of independent, identically distributed  random variables such that $(S_i,A_i) \sim \mu$, $C_i = c(S_i,A_i)$ and $S_{i}' \sim P(S_i,A_i)$.
Furthermore, \cref{eq:costnorm} holds and $c$ is a nonnegative function.
\end{assumption}
In this assumption the constraints on $C_i$ and $S_i'$ will help the learner to discover some information about the costs and the transition structure of the MDP. The independence assumption is made to simplify the analysis. \todoc{Do we know what happens if we remove this..?}
The distribution $\mu$ will be further restricted to be ``sufficiently exploratory''. To state this assumption, the following definition will be useful:
\begin{definition}[Admissible distribution]\label{defn:admissible-distributions}
    We say a distribution $\nu \in \mcM_1(\mcZZ)$ is admissible in MDP $M$
    if there exists $h \geq 1$ and a nonstationary policy $\pi$ 
    such that $\nu(s,a) = \mbP_{\pi,\eta_1}(S_h=s,A_h=a)$.
\end{definition}
Note that what constitutes an admissible distribution depends on $\eta_1$.
With this, the assumption that constrains $\mu$ is as follows:
%We introduce the necessary assumptions under which our main theorem holds.
\begin{assumption}[Finite concentrability coefficient]\label{asp:concentrability}
    There exists $C < \infty$ such that 
    for all admissible distributions $\nu$ of $M$,
	it holds that
    \begin{equation}
        \max_{(s,a)\in \cS \times \cA} \frac{\nu(s,a)}{\mu(s,a)} \leq C\,.
    \end{equation}
\end{assumption}
Note that if $\mu$ is positive over $\cS \times \cA$, the assumption is satisfied.
By only considering admissible distributions, we allow $\mu$ to be ``concentrated'' on states that are ``relevant'' in the sense that they are visited by some policy in some time step with a large probability when starting from $\eta_1$.

In addition to having access to the data, we will also assume that the learner is given access to a set of functions, $\cF$ that map state-action pairs to reals: $\cF \subset \R^{\cS \times \cA}$. 
Ideally, the set $\cF$ allows the learner to reason about the optimal action-value function. To make this possible, the following assumptions are made on $\cF$:
\begin{assumption}[Realizability]\label{asp:realizability}
    $q^\star \in \mcF$.
\end{assumption}
\begin{assumption}[Completeness]\label{asp:completeness}
    For all $f \in \mcF$ we have that $\mcT f \in \mcF$.
\end{assumption}
% \begin{assumption}[Bounded targets]\label{asp:bounded-targets}
%   There exists $B>0$ such that,
%   for all $f\in\mcF$ and $i\in[n]$, 
%   $(\mcT f)(S_i,A_i)$ and $C_i + \gamma f^\wedge(S_{i}')$
%   are in $[0,B]$ almost surely. 
%   \todos{This was a little tricky to read I think, $\mcT f$ is $\mu$-almost surely but it's not super obvious that's what's meant. I changed to $(\mcT f)(S_i,A_i)$. Also I switched the quantification around, there exists $B$ such that for all ... instead of for all ... there exists $B$. I'm pretty sure the first version wasn't what we need.}
%   \todos{Also we don't really need this to hold for all $i$, if it holds for one it will hold for all since the data are iid and it's a statement about probabilities.}
% \end{assumption}

\cref{asp:data,asp:concentrability,asp:realizability,asp:completeness} are commonly made, in various forms, when analyzing fitted Q-iteration \cite{farahmand2011regularization,Pires2012StatisticalLE,pmlr-v97-chen19e}. \cref{asp:concentrability} ensures that all admissible distributions are covered by the exploratory distribution $\mu$, i.e., that ``$\mu$ is sufficiently exploratory''. 
\cref{asp:realizability} (``realizability'')  guarantees that the optimal action-value function, our ultimate target, lies in our function class. \cref{asp:completeness}  states that the function class $\mcF$ is closed under the Bellman optimality operator $\mcT$. 
When $\cF$ is closed, completeness is easily seen to imply realizibility (see also 
footnote 10  of \citet{pmlr-v97-chen19e}). Note that \cref{asp:completeness} is necessary, this is due to a result by \citet{foster2021offline} which states that assuming both a finite concentrability coefficient and a realizable function class are not sufficient for sample efficient batch value function approximation. For a more detailed discussion of the last three assumptions, we refer the reader to Sections~4 and 5 of \citet{pmlr-v97-chen19e}.

\paragraph{Research question}
As is well known, under the above assumptions, 
and assuming that a regression oracle is available to find the empirical minimizer of regression problems
defined over $\cF$ with the squared loss,
the so-called fitted Q-iteration (FQI) algorithm
 \cite{JMLR:v6:ernst05a, riedmiller2005neural, antos-nips-2007} 
 produces a policy such that with high probability 
 $\bar v^\pi \le \bar v^\star + \tilde O(\sqrt{ C N/n } )$, 
 where $N$ is a measure that characterizes the ``richness'' of $\cF$ and
 $\tilde O$ hides logarithmic factors \cite{antos-nips-2007}.
 The main question investigated in this paper is whether this result can be improved to 
  $\bar v^\pi \le \bar v^\star + \tilde O(\sqrt{ C N \bar v^\star/n } ) + O(1/n)$. 
  The significance of such \emph{small cost} results is that the same data can produce a 
  significantly better policy when $\bar v^\star$ is near zero (note that $\bar v^\star\ge 0$).
  Alternatively, the number of samples required to achieve a given level of suboptimality can be
  significantly smaller if an algorithm satisfies a small-cost bound.
   
% \cref{asp:bounded-targets} is novel to this work and thus we discuss it separately. For our analysis, we require both inputs to log-loss, predictor and target, to be in a bounded range. Thus, we require an assumption that the function class $\mcF$ and MDP $M$ are such that for all $i\in[n]$ the targets $C_i + \gamma f^\wedge(S_{i}')$ lie in a bounded interval. Without loss of generality we let $B=1$ since any bounded domain can be normalized to $[0,1]$. 

%
%We assume the agent knows the state space $\mcS$ and the action space $\mcA$ but does not know the cost function $c$ nor the transition function $P$. The agent is given a random dataset $D = \{(S_i,A_i,C_i,S_i')\}_{i=1}^n$ and a class of candidate value functions $\mcF \subset \mbR^{\mcS \times \mcA}$. The dataset $D$ is constructed by sampling state-action pairs i.i.d. from an unknown exploratory distribution $\mu$: for each $i \in\{1,\dotsc,n\}$, $(S_i,A_i)\sim\mu$, $C_i = c(S_i,A_i)$, and $S_i' \sim P(S_i, A_i)$. Given $D$ and $\mcF$ the agent's goal is to find a near-optimal policy.
%
%  We also denote the expected value of $v^\pi$ with respect to the initial state distribution $\eta_1$ by
%  % The goal is to find a policy $\pi$ that minimizes the sum of discounted costs,
%  $\bar{v}^\pi = \innerproduct{\eta_1}{v^\pi}$, where for vectors $a,b$ of matching dimension $\innerproduct{a}{b} = \sum_i a_i b_i$ denotes the usual inner product.
%

\paragraph{Additional notation}\todoc{do we need all/any of this? Consider introducing this either earlier, or just introduce notation whenever it is needed.}
For $n \in \mathbb{N}$, let $[n]$ denote the set $\{1,2,\dots,n\}$. 
Let $\mathbb{P}_{\pi,\eta_1}$ denote the distribution 
induced over random trajectories by following policy $\pi$ after an initial state is sampled from $\eta_1$. For $h\in\mathbb{N}$, we let $\eta_h^\pi(s)$ be the probability that state $s$ is observed at timestep $h$ under $\mbP_{\pi,\eta_1}$, such that
$\eta_h^\pi(s)  = \mbP_{\pi,\eta_1}(S_h=s)$.
We also define $\eta_h^\star(s) = \eta_h^{\pi^\star}(s)$.
% Given a probability distribution $\nu$ and a probability kernel $K$, $\nu K$ will denote their composition. When using this notation we will view policies as maps to distributions over state-action pairs, $\pi : \mcS \to \mcM_1(\mcS\times\mcA)$, so that, for instance, the distribution of $(S_{h},A_{h})$ under $\mbP_{\pi,\eta_1}$ is $\eta_h^\pi\pi$, and also $\eta_{h+1}^\pi=\eta_h^\pi\pi P$. \todos{The markov kernel bit is a recent addition, not sure if it's written properly.}
For $g: \mcX \rightarrow \mbR, \, \nu \in \mcM_1(\mcX)$, and $p\geq 1$,
we define the semi-norm $\Vert \cdot \rVert_{p,\nu}$ via
$\norm{g}_{p,\nu}^p = \int |g|^p d\nu$. 
% We denote the set of non-negative reals by $\mbR_{\geq0}$.
%Since both $\mcS$ and $\mcA$ are finite, we treat any object that is defined over $\mcS$, $\mcA$ or $\mcS \times \mcA$ as a vector when convenient. \todoc{any object? what is the meaning of ``vector''? why does this need that $\mcS$ and/or $\mcA$ are finite?}
% When convenient, \todoc{this will need to be cleaned up..} we use $f \lesssim g$ to indicate that $g$ dominates $f$ up to logarithmic factors. This is done to emphasize the most salient elements of an inequality.
We adopt standard big-oh notation and write $f = \tilde{\mathcal{O}}(g)$ to denotes that $f$ dominates $g$ up to polylog factors, i.e., $f = \mathcal{O}(g\max\{1,\text{polylog}(g)\})$. 
Finally, we use $f \wedge g$ to denote $\min\{f,g\}$.
%Preliminaries

%Algo
\section{\fqilog: Fitted Q-Iteration with log-loss}\label{sec:algo}

The proposed algorithm, \fqilog,
which is described in \cref{alg:fqi},
 is based on the earlier mentioned fitted Q-iteration algorithm (FQI) \cite{JMLR:v6:ernst05a, riedmiller2005neural, antos-nips-2007}.
Given a batch dataset, FQI iteratively produces a sequence of $k$ approximations $f_1, \ldots, f_k$ to the action-value function $q^\star$.
% the algorithm in $k$ iterations uses the data to produce a sequence of action-value function estimates of $q^\star$.
 At iteration $j\in [k]$, the algorithm computes $f_j$ by minimizing the empirical loss using targets computed with the help of $f_{j-1}$, the estimate produced in the previous iteration.
The targets are constructed such that the regression function for a fixed estimate $f \in \mcF$ is $\mcT f$.
The main difference between the proposed method and 
the most common variant of FQI is our use of log-loss, 
\begin{equation}\label{eqn:log-loss}
  \elllog(y,\tilde y) =  \tilde y\log\frac{1}{y} +
  (1 - \tilde y)\log\frac{1}{1-y} \,,
\end{equation}
to measure the deviation between a prediction ($y\in [0,1]$) and a target $(\tilde y\in [0,1]$),
where to allow $y,\tilde y\in \{0,1\}$, we use $0\cdot \log 0 = 0$.
The restriction that $\tilde y\in [0,1]$  means that $\elllog(\cdot,\tilde y)$ is convex.
In our algorithm the first argument of $\elllog$ is a predicted value $f(S_i,A_i)$ with $f\in \cF$. 
Since this needs to also belong to $[0,1]$,
running \fqilog requires the range of all functions in $\cF$ to lie in $[0,1]$. 
\todoc{it is unclear to me whether we need a restriction on $\tilde y$; per the next paragraph we need that whatever random response goes in place of $\tilde y$ has an expected value in $[0,1]$. But do we need it is in $[0,1]$? Maybe convexity is used somewhere in the analysis?}
%\todoa{Dylan has that the predictor $y$ is in $[0,1]$. Cs: But does their analysis use this?}
%\todos{If $\tilde y$ is outside of $[0,1]$ then the loss could become negative. We also need $\tilde y$ to have nonnegative expectation for the expectation to be the minimizer in the next paragraph, and I think the expectation should be $\le 1$.}

Note that previous work on FQI employed the squared loss $\ellsq : \mbR \times \mbR \to \mbR$, defined as $\ellsq(y,\tilde y) = (y-\tilde y)^2$, instead of the log-loss $\ellog$.
Both squared loss and log-loss have the property that for a random variable $X$ taking values in $[0,1]$ with probability one, $\EE[X] = \argmin_{m\in \R} \EE[\ell(m,X)]$, where $\ell$ is either $\ellog$ or $\ellsq$. In particular, if $\cF = [0,1]^{\cS \times \cA}$, both log-loss and squared loss will give rise to the same sequence $f_1,\dots,f_k$.
% 0 = d/dy \ell(y,m) = m y (-1/y^2) + (1-m) (1-y) (-1/(1-y)^2) (-1) = 
% = -m/y+(1-m)/(1-y) 
% <==> 
% 0 = y-m y - (1-y)m = y - my - m + my <=> m = y
% we also need that we get a minimizer; this is only true if 0<=m<=1.

The differences between log-loss and squared loss are made apparent when $\cF\ne [0,1]^{\cS \times \cA}$.
In this case, when $\tilde y$ is near $0$ or $1$, $\elllog(\cdot,\tilde y)$ increases rapidly as $y$ deviates from $\tilde y$.
As such, when errors need to be traded off at different inputs, log-loss will end up 
favoring predictors that predict values closer to observed targets when the targets are near $0$ (or $1$),
and will put less weight on observed targets in the middle of the $[0,1]$ range.
When a true value lies near $0$ (or $1$), the observed value (bound to the range $[0,1]$)
must be closer to the true value, which means that the observed value is also close to $0$ (or $1$).
While it may happen that an observed value is close to $0$ (or $1$) while its mean is far from it,
this is rare: For this to happen, the observed value has to have a large variance.
As such, favoring  to predict observed values near the $0$ or $1$ as opposed to paying equal attention to all datapoints (which is what the squared loss based predictors do) is beneficial, and, in particular it pays off when some of the true targets are near $0$ (or $1$): The situation that arises when the optimal cost is near zero.

The motivation to switch to log-loss is due to
\citet{foster2021efficient} who studied the problem of learning a near-optimal policy in contextual bandits,
both in the batch and the online settings.
%in the context of learning a good policy in contextual bandits in an online setting by learning a reward model.
They noticed that switching to log-loss from squared loss allows bounding the suboptimality of the policy found, say in the batch setting after seeing $n$ contexts, via a term that scales with  $\sqrt{\bar v^\star/n} + 1/n$. 
This is an improvement from the usual $\sqrt{1/n}$ bound derived when analyzing squared loss, which is worst-case in nature.
For log-loss, a significant speedup to $1/n$-type convergence is achieved when $\bar v^\star$, the expected cost of using the optimal policy, is small (cf. Section 3.1 of their paper). They complemented the theory with convincing empirical demonstrations.
Our results take a similar form. While we reuse some of their results and techniques, our analysis deviates significantly from theirs. In particular, our analysis must be adapted to handle the multistage structure present in RL and to avoid an unnecessary dependence on the actions.

%Since our problem is a multistage and their bounds depend on the number of actions, we adapt our analysis to handle multistage structure and avoid a dependence on the number of actions.

The astute reader may wonder whether switching to log-loss is really necessary for achieving small-cost bounds. As it turns out, the switch is necessary, as attested to by an example constructed by  \citet{foster2021efficient}. In this example, in contrast to log-loss, squared loss is shown to be unable to take advantage of small optimal costs (cf. Theorem 2 of \citet{foster2021efficient}). 

%% OLD TEXT END %%%%%%%%%%%%%%%%%%%%%%%%%%%%%%%%%%%%%
%%%%%%%%%%%%%%%%%%%%%%%%%%%%%%%%%%%%%%%%%%%%%%%%%%

\begin{algorithm}[tb]
   \caption{\fqilog} %Fitted Q-Iteration (FQI)}
   \label{alg:fqi}
\begin{algorithmic}
   \STATE {\bfseries Input:} A dataset $D_n=\{(S_i,A_i,C_i,S_{i}')\}_{i\in[n]}$, a function class $\mcF \subseteq [0,1]^{\mcS\times\mcA}$ and $k \in \mathbb{N}$.
   \STATE Pick $f_0$ arbitrarily from $\cF$
   \FOR{$j=1$ {\bfseries to} $k$}
%   \FOR{$i=1$ {\bfseries to} $n$}
%   \STATE $f_{\kappa-1}^\wedge(s_{i}) \leftarrow \min_{a\in\mcA}f_{\kappa-1}(s_{i},a)$.
%   \ENDFOR
   % \STATE $Y_i \leftarrow C_i + \gamma f_{j-1}^\wedge(S_i')$
   \STATE $f_j \leftarrow \argmin\limits_{f \in \mcF}\sum_{i=1}^n \elllog\left(f(S_i,A_i), C_i + \gamma f_{j-1}^\wedge(S_i')\right)$
   % \STATE $f_j \leftarrow \argmin_{f \in \mcF}\sum_{i=1}^n \elllog\left(f(S_i,A_i);\, Y_i\right)$,
   % where $Y_i = C_i + \gamma f_{j-1}^\wedge(S_i')$. % old version uses the where
   \ENDFOR
%   \FOR{$s \in \mcS$}
%   \STATE $\pi_{f_k}(s) \leftarrow \arg\min_{a\in\mcA} f_k(s,a)$.
%   \ENDFOR
   \STATE {\bfseries Return:} $\pi_{f_k}$.
\end{algorithmic}
\end{algorithm}
%Algo

%theory

\section{Theoretical Results}\label{sec:theory-results}
In this section, we present our main theoretical contribution, the first \textit{small-cost} bound for an efficient algorithm in batch RL.
%\subsection{Main Theorem}
\begin{theorem}\label{thm:main-paper}
  \todos{will we use $D$ or $D_n$? Both are used in different places and I have no strong preference. CS; $D_n$.}
  Given a dataset $D_n = \{(S_i,A_i,C_i,S_{i}')\}_{i=1}^n$ with $n\in\NN$ and a finite function class $\mcF \subseteq [0,1]^{\mcS \times \mcA} $ that satisfy \cref{asp:data,asp:concentrability,asp:realizability,asp:completeness},
   it holds with probability $1-\delta$ that 
   the suboptimality
   gap $g = \bar{v}^{\pi_{k}} - \bar{v}^\star$ of 
   the output policy of \fqilog after $k$ iterations, $\pi_{k} = \pi_{f_k}$, satisfies 
    \begin{align*}
%        \MoveEqLeft\bar{v}^{\pi_{k}} - \bar{v}^\star \\
g        &\leq \tilde{\mathcal{O}}\left( \frac{1}{(1-\gamma)^2}\left(\sqrt{\frac{\bar{v}^\star C N}{n}} + \frac{C N}{(1-\gamma)^2 n} +\gamma^k\right)\right)\,,
    \end{align*}
    where $N = \log(|\mcF|/\delta)$ and $C$ is defined in \cref{asp:concentrability}.
%\footnote{We note in passing that \cref{asp:realizability} is implied by \cref{asp:completeness} when $\cF$ is closed under the topology induced by max-norm, which is true when $\cF$ is finite.}    
    \todoc{missing function set complexity terms}
  \end{theorem}
The full statement of \cref{thm:main-paper}, including lower order terms, can be found in \cref{sec:appendix-proof} along with its proof. Compared to prior error bounds for FQI \cite{antos-nips-2007,antos2008learning,JMLR:v9:munos08a,farahmand2011regularization,lazaric2012finite,pmlr-v97-chen19e}, to the best of our knowledge, \cref{thm:main-paper} is the first that contains the instance-dependent optimal cost $\bar v^\star$.
This makes \cref{thm:main-paper} a small-cost bound, also referred to as a \textit{first-order} \cite{freund1997decision,pmlr-v40-Neu15} or \textit{small-loss} \cite{doi:10.1287/moor.2021.1204,wang2023benefits} bound in the learning theory literature. 
All previous results for FQI obtain an error bound independent of $\bar v^\star$, and cannot be made to scale with $\bar v^\star$ due to their use of squared loss (see Theorem 2 of \citet{foster2021efficient}, mentioned earlier). Finally, we highlight that \cref{thm:main-paper} is the first small-cost bound for a batch RL algorithm that is \emph{computationally efficient} when efficient regression oracles are available, as is the case when $\cF$ is the set of logit models with weights bounded in $2$-norm.
While technically this is outside of the scope of \cref{thm:main-paper} (since in its current form this result  covers only finite model classes), with some extra work and with appropriate modifications one can show that  \cref{thm:main-paper} continues to hold for infinite model classes, such as the mentioned logit class.

\subsection{Proof Sketch}
The purpose of this section is to give a sketch of the proof of \cref{thm:main-paper}.
For the full proof, see  \cref{sec:appendix-proof}. 
We start by defining the pointwise triangular deviation of $f$ from $q^\star$,
$$
\Delta_f^2(s,a) = \frac{(f(s,a) - q^\star(s,a))^2}{f(s,a) + q^\star(s,a)}\,,
$$
which is closely related to triangular discrimination \cite{850703}. We can relate $\Delta^2_f$ to the Hellinger distance via the following lemma:
\begin{lemma}\label{prop:main-paper}
    For all $p,q \in [0,1]$, we have 
    \begin{equation}
        \frac{1}{4} \frac{(p-q)^2}{p+q} \leq \frac{1}{2} \left(\sqrt{p}-\sqrt{q}\right)^2 
        \leq \hell^2(p\, \| \, q)\,,
    \end{equation}
    where for $p=q=0$ we define the left-hand side to be zero and $\hell^2(p\Mid q)= \frac12(\sqrt{p} - \sqrt{q})^2 + \frac12 (\sqrt{1-p} - \sqrt{1-q})^2$ is the squared Hellinger distance. \todoc{big bug was here.. watch out guys!}
\end{lemma}
The proof of \cref{prop:main-paper} is deferred to \cref{sec:basicineq}.
%\cref{prop:hellinger-triangular-scalar}.
The idea to relate the pointwise triangular deviation to the squared Hellinger distance was first employed by \citet{foster2021efficient} in analyzing regret bounds for contextual bandits.
Our proof can be summarized by the following three main steps, which correspond to the three terms given in \cref{prop:main-paper}.

\paragraph{Step 1: Error decomposition} The first step in the proof is to decompose the error (or suboptimality gap), $\bar{v}^{\pi_k}-\bar{v}^\star$, into the product of a small-cost term and the pointwise triangular deviation of $f$ from $q^\star$. The analysis in this step is inspired by the proof of Lemma 1 of \citet{foster2021efficient}. We deviate from their analysis to avoid introducing an extra $|\mcA|$ factor in the bound. We use the performance difference lemma (\cref{lem:perf-diff}), a \textit{multiplicative} Cauchy-Schwarz (\cref{lem:triangular-discrim-bound}), i.e. for distribution $\nu$
$$
\norm{x-y}_{1,\nu} \leq \norm{x+y}_{1,\nu}^{1/2}\cdot \norm{\frac{(x-y)^2}{(x+y)}}_{2,\nu}\,,
$$
and an implicit inequality (i.e. \cref{lem:bound-sum-f-fstar}, step $\star$ in the proof of \cref{lem:first-order-decomposition2}) to get a small-cost decomposition of the error:
\begin{proposition}\label{lem:first-order-decomposition2-main}
    Let $f: \mcS \times \mcA \rightarrow [0,\infty)$  and let $\pi = \pi_f$ be a policy that is greedy with respect to $f$. 
    Define $D_f = \sup_{h\ge 1} \max(\norm{\Delta_f}_{2,\nu_{1,h}},\norm{\Delta_f}_{2,\nu_{2,h}})$. 
        Then, it holds that
        \begin{align*}
    %        \MoveEqLeft
            \bar v^{\pi}-\bar v^\star\leq 
    \tilde{C}\left(\frac{ D_f}{1-\gamma}\sqrt{ \bar{v}^\star} + \frac{D_f^2}{(1-\gamma)^2}\right)\,. 
        \end{align*}
        where $\tilde{C} > 0$ is an absolute constant. 
\end{proposition}

Here $\nu_{1,h},\nu_{2,h} \in \mcM_1(\mcS\times\mcA)$ are appropriately defined distributions, the details of which can be found in \cref{lem:first-order-decomposition2}. In summary, step 1 uses the pointwise triangular deviation of $f$ from $q^\star$ to bound the error by the optimal value function $\bar{v}^\star$.

\paragraph{Step 2: Contraction} The second step in our proof starts by bounding the pointwise triangular deviation by the Hellinger distance, i.e.
\begin{align*}
    \frac{1}{\sqrt{2}}\norm{\Delta_f}_{2,\nu} \leq \norm{\sqrt{f}-\sqrt{q^\star}}_{2,\nu}
\end{align*}
where $\nu \in \mcM_1(\mcS\times\mcA)$. In \cref{lem:contraction-fixed} we next establish that $\mcT$ is a $\gamma$-pseudo-contraction at $q^\star$ with respect to Hellinger distances: For any $f\ge 0$, $\nu\in \cM_1(\cS \times \cA)$,
$$
\norm{\sqrt{\mcT f}-\sqrt{\mcT q^\star}}_{2,\nu} \leq \gamma^{1/2}\norm{\sqrt{f}-\sqrt{q^\star}}_{2,\nu'}\,,
$$
where $\nu'$ is a distribution over state-action pairs. 

Contraction arguments have a long history \cite{bertsekas1995,littman1996algorithms,antos-nips-2007,pmlr-v97-chen19e} in the analysis of dynamic programming algorithms that solve MDPs. The novelty is that we needed to bound the pointwise triangular deviation, which led to new analysis with Hellinger distances. 
%Note that such change of measure arguments used in the ``error propagation analysis'' of RL algorithms are standard. 

In \cref{lem:contraction},  the combination of a standard change of measure argument (that uses the definition of the concentration coefficient $C$) and a standard contraction argument, we get
\begin{equation}\label{eqn:temp}
\norm{\sqrt{f}-\sqrt{q^\star}}_{2,\nu} \leq \frac{\sqrt{C}}{1-\sqrt{\gamma}}\norm{\sqrt{f} - \sqrt{\mcT f}}_{2,\mu}\,,
\end{equation}
where $\mu$ is the exploratory distribution in \cref{asp:data}. In batch RL, we often want to learn $q^\star$, or a policy whose value is close to $q^\star$. However, when using function approximation FQI, the regressor corresponding to the training data is $\mcT f$. Therefore, as demonstrated in \cref{eqn:temp}, if we can show the contraction property then we can bound the error between $f$ and $q^\star$ by the error between $f$ and $\mcT f$. As we will argue shortly, the error between $f$ and $\mcT f$ goes to zero as the size of the batch dataset grows.

\paragraph{Step 3: Error control/propagation} The third step in our proof starts by bounding 
$$\norm{\sqrt{f}-\sqrt{\mcT f}}_{2,\mu} \leq \sqrt{2}\norm{\hell^2(f \Mid \mcT f)}_{1,\mu}^{1/2}\,.$$ 
Then by application of \cref{thm:concentration}, we have that if $f$ is the minimizer of $\elllog$ with respect to the batch dataset $D_n$, then
$$
\sqrt{2}\norm{\hell^2(f \Mid \mcT f)}_{1,\mu}^{1/2} \lesssim \frac{2}{\sqrt{n}}\,,
$$
where we use $\lesssim$ informally to highlight the most salient elements of the inequality.
We then combine all the results to control the pointwise triangular deviation in \cref{lem:first-order-decomposition2-main}, i.e $D_f$, 
$$
\bar{v}^{\pi_k} - \bar{v}^\star \lesssim \frac{\sqrt{C\bar v^\star}}{(1-\gamma)^2 \sqrt{n}} + \frac{C}{(1-\gamma)^4  n}\,.
$$

%The final step is to makes use of an implicit inequality to get our small-loss bound. Notice that by the definition of $\bar{v}^\star$ we have that $\bar{v}^\star \leq \bar{v}^{\pi_k}$. Fixing $h$, this allows us to bound $\norm{v^\star}^{1/2}_{1,\eta_{3,h}} \leq \norm{v^{\pi_k}}^{1/2}_{1,\eta_{3,h}}$. An application of \cref{lem:non-negativity-vpi} gives 
%$$
%\sum_{h=1}^\infty \gamma^{h-1}\norm{v^{\pi_k}}^{1/2}_{1,\eta_{3,h}} \leq \sum_{h=1}^\infty \sqrt{\gamma^{h-1} \bar{v}^{\pi_k}}\,.
%$$
%Now that we have that bounded $\bar{v}^{\pi_k}$ by its root $\sqrt{\bar{v}^{\pi_k}}$ we can make use of an implicit inequality, i.e. for $x,b,c$ nonnegative $x \leq \sqrt{cx} + b \implies x \leq 2b + c$ via AM-GM, to get a small-loss bound, where $x=\bar{v}^{\pi_k}$ and $b=\bar{v}^{\star}$. Plugging in this bound on $\bar{v}^{\pi_k}$ completes the proof sketch.

This provides a sketch for proving \cref{thm:main-paper}. In the full proof we also need to control each iterate of \fqilog, i.e. the $f_j$'s. We fill in the missing details in our appendix.

%theory

%numerics
%!TEX root =  ../paper.tex
\section{Numerical Experiments}\label{sec:numerical-results}
The goal of our experiments is to provide insights into the benefits of using \fqilog for learning a near-optimal policy in batch reinforcement learning. 
We run our first set of experiments in reinforcement learning with logit models, as 
\emph{this setting allows us to best compare \fqilog to \fqisq without other confounding factors}. For these experiments, we used two standard control tasks; ``mountain car'' and ``inverted pendulum''. The tasks are set up as fixed horizon episodic problems where at the end of an episode, if the goal is met no cost is incurred, otherwise a cost of one is incurred. The two environments differ in that in one of them the goal region is small, in the other the goal region is large. In both tasks, some policies can reach the goal, but many fail.
As it is known that the goals in these problems can be met, we expect \fqilog to do better than \fqisq on these environments.

Our second set of experiments aim at verifying whether the recommendation to switch to log loss transfers
to deep RL (DRL), i.e., to more complex function classes, regression methods and environments. 
For these experiments, we started from the work of \citet{agarwal2020optimistic}, 
who tested various DRL methods, including C51
of \cite{bellemare2017distributional}, a distributional RL algorithm, which was found to be one of the most capable of the methods tested.
As noted in the introduction, \citet{wang2023benefits} showed a small-cost bound for a distributional RL method; hence, our research question is whether a simpler log-loss based method can compete with these (more complex) distributional RL algorithms. To create a real challenge, we picked the two environments (Asterix and Sequest) from \citet{agarwal2020optimistic} where C51 
 significantly outperformed \dqnsq, the DRL version of FQI and copied their setting.
%
%
%we report the cumulative undiscounted reward, as is standard in deep reinforcement learning \cite{mnih2015humanlevel}, on Asterix and Seaquest as a function of the number epochs. 
%The batch datasets for both Asterix and Seaquest are the $5$ batch datasets taken from \citet{agarwal2020optimistic}. 
%The objective is either to minimize squared loss, which does not get small-cost bounds, or to minimize log-loss, as proposed in this work. 

%\subsection{Measurements}
%In the first set of experiments, we report the value of the learned policy both as a function of the number of training samples and the number of ``good'' trajectories in the batch dataset to indicate the sample efficiency of the two algorithms. The results for mountain car are obtained from $90$ independently collected batch datasets, generated by following the uniform random policy. 

\subsection{Aiming for a Goal: Mountain Car}

\begin{figure*}
    \centering
    \includegraphics[width=0.33\linewidth]{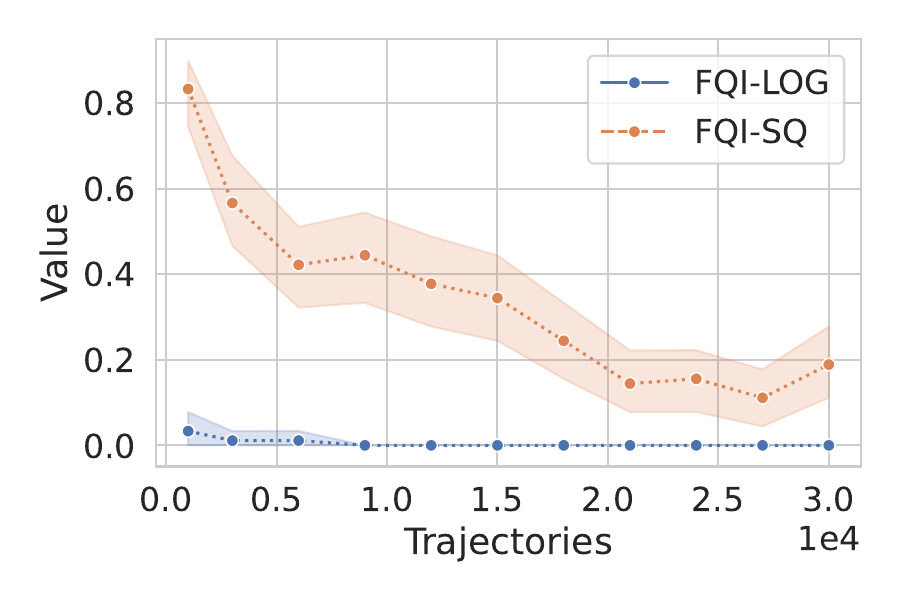}
    \includegraphics[width=0.33\linewidth]{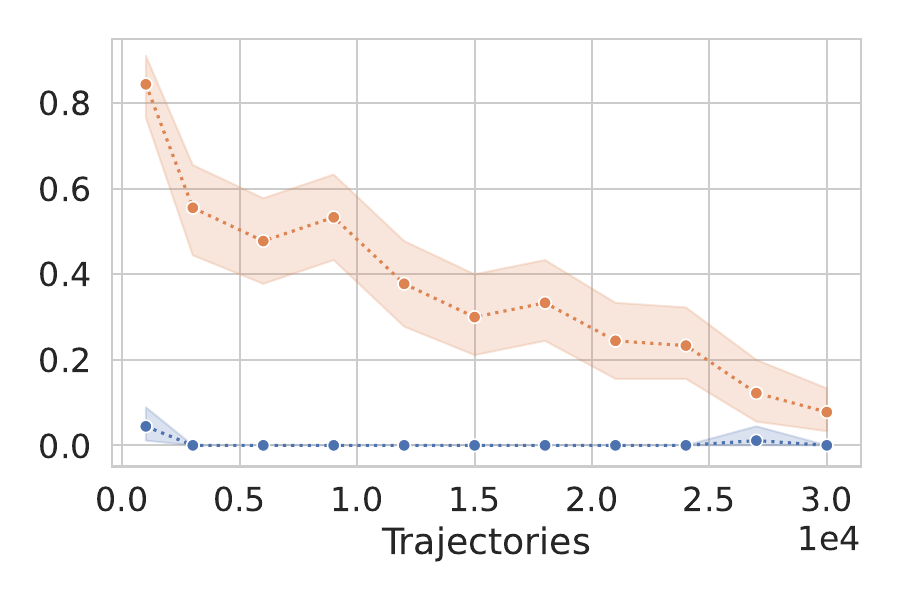}
    \includegraphics[width=0.33\linewidth]{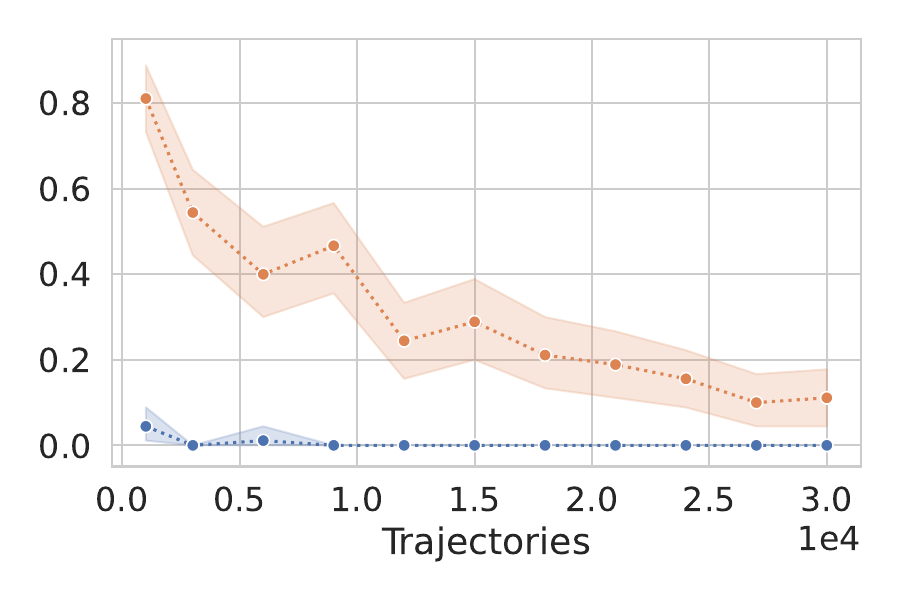}
    \caption{The value of the policy learned by FQI as a function of the size of batch dataset. The results are averaged over $90$ independently collected datasets. The figures on the left, middle and right are generated using batch data that contain only $1$, $5$, and $30$ successful trajectories respectively. The standard error of the mean is reported via the shaded region.}
    \label{fig:mc-plots}
\end{figure*}

We first evaluate \fqilog and \fqisq on an episodic sparse cost variant of mountain car with episodes lasting for $800$ steps.
(While we showed our results for the discounted setting, they are expected to hold in episodic problems as well, with small modifications.)
Following \citet{Moore90efficientmemory-based}, this environment consists of a $2$-dimensional continuous state space of $[-1.2,0.6]\times [-0.07,0.07]$
and $3$ discrete actions;
the states represent a position and velocity of an underpowered car that can be accelerated left, right, or not accelerated, until the top of a hill is reached when the dynamics is turned on, and the car remains in place regardless of the actions.
The cost is 0 at all timesteps except the last, when a cost of 1 is received if the learner has not reached the hilltop. We consider the undiscounted version of the problem (i.e., $\gamma=1$). An optimal policy for this setting reaches zero cost if it reliably reaches the top of the hill in $800$ steps or less, regardless of the exact time. For $\eta_1$, the initial state distribution, we use a Dirac that outs the car at the bottom of the hill with zero velocity with probability one. 
%
%The learner, i.e. the policy $\pi_{f_k}$, starts from a fixed state, $s_1 = [-0.5,0]$, and interacts with the environment for 800 timesteps before an episode terminates. If an learner reaches the top of the hill, it remains there until the end of the episode. The cost is 0 at all timesteps except the last, when a cost of 1 is received if the learner has not reached the hilltop, and otherwise the cost is 0. A discount factor of $\gamma = 1$ is used in our experiments.

The feature vectors assigned to states are $16$ dimensional and come from
a Fourier basis of order $4$, following \citet{konidaris2011value}  and Chapter 9 of \citet{sutton2018reinforcement}. 
With this, for time step $h\in [800]$, the estimator uses $\theta_h = (\theta_a^h)_{a\in \cA} \in \R^{48}$
to produce the estimate $f_h(s,a) = \sigma( \ip{\phi(s),\theta_a^h})$,
%Given these features our function class $\mcF$ for this environment is 
%$$
%\left\{f_h: \mcS \times \mcA \rightarrow [0,1] : \exists \theta_{a}^h \  \text{ s.t. } \  f_h(s,a) = \sigma\innerproduct{\phi(s)}{\theta_{a}^h}\right\},
%$$
where $\sigma(x) = (1+\exp(-x))^{-1}$ is the sigmoid function. 
This variant of \fqilog with sigmoid functions is closely related to the logistic temporal-difference learning algorithm proposed in Appendix A of the PhD thesis of \citet{silver}. We employ the BFGS method, a quasi-Newton method with no learning rate, to find the minimizer of the losses. 
While for strongly-convex functions, BFGS is known to converge to the global minimum superlinearly \cite{dennis1974characterization}, there is no guarantee that BFGS will find a global minimum when the loss is the squared loss.
Finally, each batch dataset is constructed from a set of trajectories collected by running the uniform random policy from the initial state $30,000$ times. We use rejection sampling in order to guarantee each dataset has $i$ trajectories that reach the top of the hill with 
$i\in \{1,5,30\}$. 
We train both \fqilog and \fqisq on the same batch datasets with the first $n = [1,3,6,9,12,15,18,21,24,27,30] \times 10^3$ trajectories in order to study the relationship between the size of the batch dataset and the quality of the policies learned by \fqilog and \fqisq. 
% We move the successful trajectories to the front of the batch dataset otherwise the result would be trivial if all the trajectories in the batch dataset set failed to reach the top of the hill.
We move the successful trajectories to the front of the batch dataset so that they are included for all values of $n$.

Since BFGS does not require tuning and the same batch datasets are given to \fqilog and \fqisq, the only variable effecting the performance of the two methods is the loss being minimized. As shown in \cref{fig:mc-plots}, \fqilog accumulates much smaller cost then \fqisq using fewer samples, irrespective of the number of trajectories that reach the top of the hill. \fqilog is also able to learn a near-optimal policy with only a \textbf{single} successful trajectory. In batch RL collecting good trajectories is often expensive, so making efficient use of the few that appear in the batch dataset is an attractive algorithmic feature. As the number of trajectories that reach the top of the hill increases, so too does the performance of \fqisq. Since the optimal value on this problem is zero, \fqilog is able to learn a near-optimal policy using fewer samples than \fqisq.

\subsection{Avoiding Failure: Inverted Pendulum}

We further evaluate \fqilog and \fqisq on the inverted pendulum environment \cite{lagoudakis,riedmiller2005neural}, where the goal is to balance an inverted pendulum, by applying forces to it. 
The state space is two dimensional (angle, angular velocity)
and there are three actions (push left, right, no change).
The environment dynamics are as described by \citet{lagoudakis}, except that 
{\em (i)} when the pendulum falls below horizontal, the state is frozen there and 
{\em (ii)} we clip the angular momentum to be in $[-5,5]$ instead of letting it take arbitrary real values;
the clipping is done to facilitate the use of Fourier basis features, which we found works better than the radial basis functions used by \citet{lagoudakis}. As we were using order $4$ Fourier features, the parameter space in this case is $4\times 4\times 3 =48$ dimensional. The same logit class was used as in the case of mountain car; for fitting BFGS was used.
The cost structure is as follows: The cost of letting the pendulum fall bellow the horizontal is $0$. There is no cost otherwise. Again, we know there exist policies that achieves near-zero cost.
The discount factor is $\gamma=0.95$.
All datasets are collected by a policy that selects actions uniformly at random until failure (which happens typically after $6$ steps) starting from a close to upright, random position. 
We performed $90$ independent trials, each trial consisting of fitting \fqilog and \fqisq on the same data for $k=300$ rounds. We then evaluate the learned policies $1000$ times on whether the inverted pendulum is still balanced after $3000$ steps, starting from a near upright, random position. Results are shown in \cref{fig:invpenplot}. As with the mountain car experiments, we see that \fqilog uses fewer samples than \fqisq to learn a good policy. 
\begin{figure}[htb]
    \centering
    \includegraphics[width=0.9\linewidth]{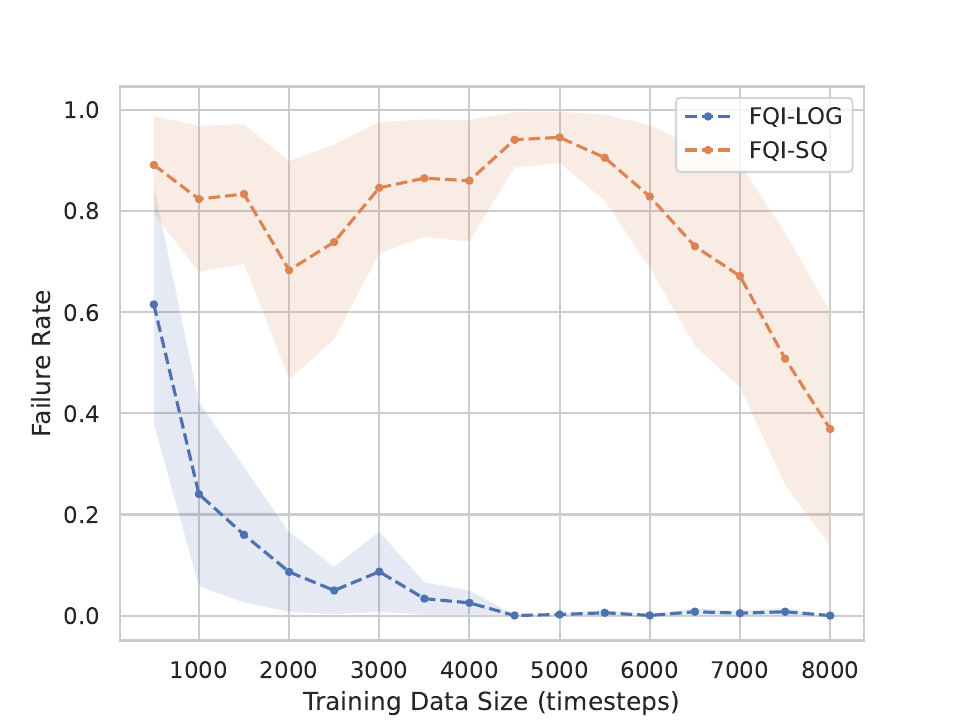}
    \caption{The portion of the time that a policy learned by FQI was able to balance the pendulum for $3000$ steps. 
Results are averaged over $90$ independently collected datasets and each learned policy is tested on $1000$ initializations. The standard error of the mean is shaded.}
    \label{fig:invpenplot}
\end{figure}

\subsection{Asterix and Seaquest}
As described earlier, 
we evaluate the deep RL variants of \fqilog and \fqisq on the Atari 2600 games Asterix and Seaquest \cite{bellemare2013arcade}, and use the distributional RL algorithm C51 as an additional baseline.
We adopt the data and experimental setup of \citet{agarwal2020optimistic}. The data consist of five batch datasets for each game, which were collected from independent training runs of a DQN learner \cite{mnih2015humanlevel}. Specifically, each batch dataset contains every fourth frame from 200 million frames of training; a frame skip of four and sticky actions \cite{machado2018revisiting} were used, whereby all actions were repeated four times consecutively and a learner randomly repeated its previous action with probability 0.25.
% Following the experimental setup of \citet{agarwal2020optimistic}, we use their $5$ batch datasets, which were collected by training DQN \cite{mnih2015humanlevel} agents each with $200$ million frames, a frame skip of $4$ and sticky actions \cite{machado2018revisiting}, where with $25$ probability the agent takes the previous action.

When the function class $\mcF$ of \fqisq is given by a deep neural network, the algorithm is called DQN. To adapt \fqilog to the deep RL setting, we must switch the training loss from $\ellsq$ to $\elllog$, and add a sigmoid activation layer to squash the output range to $[0,1]$. We henceforth refer to these algorithms as \dqnsq and \dqnlog, respectively.

% The deep RL variant of \fqisq is DQN \cite{fan2020theoretical}. DQN uses a neural network optimized with the squared loss in order to fit $q$-values. The deep RL variant of \fqilog is also a DQN, implemented with \todoa{binary} logarithmic loss and an added sigmoid activation layer. For our experiments we refer by \dqnlog and \dqnsq the variants of DQN trained with the \todoa{binary} log and squared loss respectively. 

The first algorithm to implement a variant of a DQN trained with a form of log-loss is the distributional RL algorithm C51, i.e. categorical DQN \cite{bellemare2017distributional}. C51 minimizes the \textit{categorical} log-loss across $N$ categories:
$$
\ell_{\log,N}(y;\tilde y) = \sum_{i=1}^N \elllog(y_i ; \tilde y_i)\,,
$$
for $y,\tilde y \in [0,1]^N$.
C51 modifies \dqnsq in the following five ways:
\begin{enumerate}[label=\textbf{S}.\arabic*]
 \item\label{list:bins} C51 categorizes the return, i.e. sum of discounted rewards, into 51 ``bins'', and predicts the probability that the outcome of a state-action pair will fall into each bin, whereas \dqnsq regresses directly on the returns.
   % C51 categorizes the returns, i.e. sum of discounted rewards, into $51$ ``bins'' and assigns probabilities to the outcomes of each bin whereas DQN regresses directly on the returns.

 \item\label{list:softmax} C51 applies a softmax activation to its output, to normalize the values into a probability distribution over bins, as necessitated by \cref{list:bins}.
   % to represent the normalized values whereas DQN does not use a softmax (or sigmoid) layer.

 \item\label{list:c51-loss} C51 exchanges $\ellsq$ for $\ell_{\log,N}$ as the training loss.
   % \item\label{list:c51-loss} C51 minimizes the \todoa{categorial} log-loss, $\ell_{\log,N}$, whereas DQN minimizes the squared loss, $\ellsq$.

 \item\label{list:c51-clip} C51 ``clips'' the targets to the finite interval $[v_{\min}, v_{\max}]$, to enable mapping them into a finite set of bins.
% \item\label{list:c51-clip} C51 ``clips'' the targets that are less than $V_{\min}$ or greater $V_{\max}$ to be $V_{\min}$ and $V_{\max}$ respectively, i.e. $\max\{V_{\min},\min\{c + \gamma f(s'), V_{\max}\}\}$ whereas DQN does not clip the targets.  

    \item\label{list:c51-distrib-bellman} C51 replaces the Bellman optimality operator $\mcT$ with a modified ``distributional Bellman operator''.
    % \item\label{list:c51-distrib-bellman} C51 employs a distributional Bellman update whereas DQN employs the Bellman optimality operator $\mcT$ for its updates.
    \end{enumerate}

    For our experiments we clip the targets of \dqnlog by setting $v_{\min} =0$ and $v_{\max} = 10$. Since the sigmoid activation of \dqnlog is a specialization of the softmax activation to the binary case, \dqnlog implements the changes \ref{list:c51-loss}, \ref{list:softmax}, and \ref{list:c51-clip} to the standard form of \dqnsq.
    Clipping the targets introduces a bias which we correct for by similarly clipping the targets of \dqnsq. Thus our benchmark results for \dqnsq include the change \ref{list:c51-clip}. Clipping the targets of \dqnsq is novel to this work and improves performance, yielding a stronger baseline. We include a comparison with the traditional unclipped version of \dqnsq in \cref{sec:appendix-experiment}.

%    \citet{agarwal2020optimistic} found that C51 significantly outperforms \dqnsq on the games Asterix and Seaquest. We selected these games to test our hypothesis that the improved performance of C51 can be, in part, attributed to the choice of loss function.
    In our implementations of \dqnlog and \dqnsq, we use the same hyperparameters reported by 
\citet{agarwal2020optimistic}. \cref{fig:atari} shows the undiscounted return as a function of the number of     training epochs. On Seaquest, \dqnlog outperforms \dqnsq and matches the performance of C51. In 
Asterix, \dqnlog performs similarly to \dqnsq and both get lower return than C51.
Overall, our results are inconclusive in this setting in regards to whether switching to log-losses suffices to reproduce the success of C51. However, the experiments confirm that switching to log-loss can be beneficial, as compared to using the squared loss and sometimes this switch alone is sufficient to compete with the more complex C51 algorithm.

\begin{figure*}[t]
    \centering
    \includegraphics[width=0.4\linewidth]{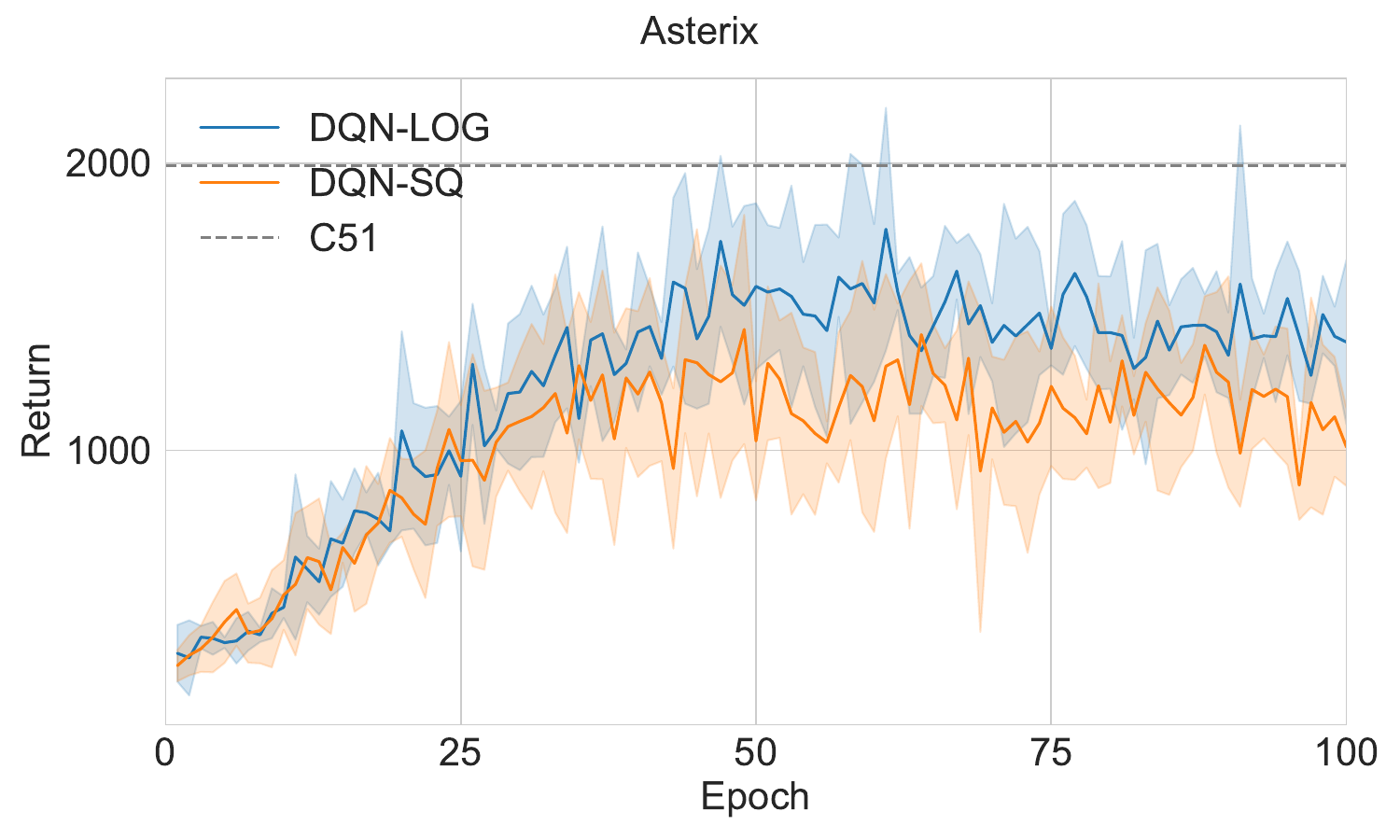}
    \includegraphics[width=0.4\linewidth]{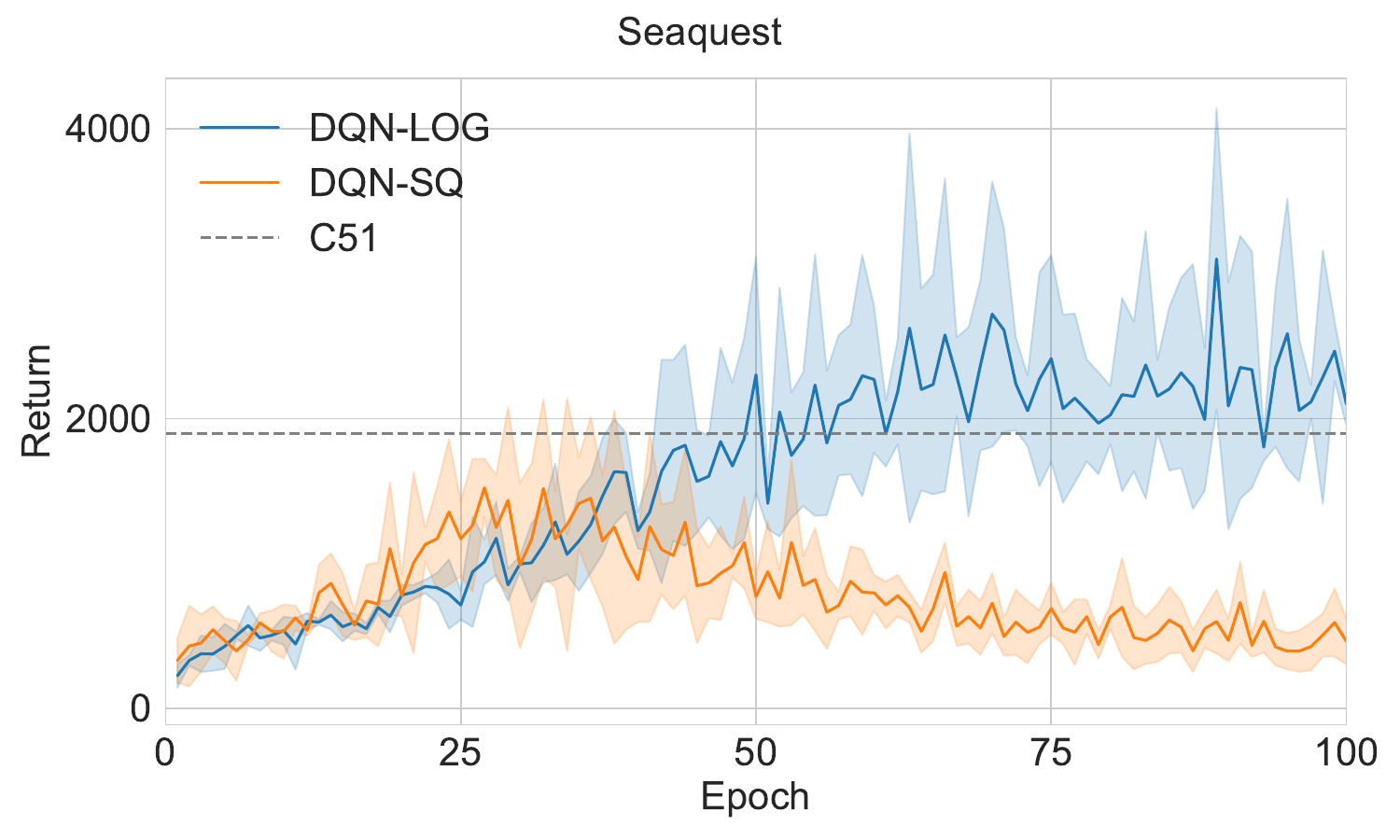}
    \caption{Learning curves on Asterix and Seaquest. The results are averaged over $5$ datasets. The shaded regions represent one standard error of the mean. One epoch contains 100k updates.}
    \label{fig:atari}
\end{figure*}

%numerics

%related works
\section{Related Works}
\paragraph{First-order bounds in RL} \citet{wang2023benefits} obtain small-cost bounds for finite-horizon batch RL problems under the distributional Bellman completeness assumption, which is more restrictive than our analogue, \cref{asp:completeness}. \citet{wang2024more} refines the bounds of \citet{wang2023benefits}, showing second-order bounds (which depends on the variance) for the same algorithm. They attribute their small-cost bound to the use of the distributional Bellman operator \citep{bdr2023}. However, their proof techniques \textit{only} make use of pessimism \citep{buckman2021the,jin2021pessimism} and log-loss in achieving their small-cost bound. The use of pessimism is necessary for their proof in order to control the errors accumulated by use of the distribution Bellman operator during value iteration. We improve upon their work by proposing an efficient algorithm for batch RL that enjoys a similar small-cost bound without using the distributional Bellman operator or pessimism under a \textit{weaker} completeness assumption.

\citet{jin2020reward} and \citet{pmlr-v162-wagenmaker22b} obtain regret bounds that scale with the value of the optimal policy. However in their setting, the goal is to maximize reward. Therefore, their bounds only improve upon  previous bounds (e.g., \citealt{azar2017minimax,pmlr-v97-yang19b,pmlr-v125-jin20a})
when the optimal policy accumulates very little reward. These bounds are somewhat vacuous as they only imply that regret is low when the value of the initial policy is already close to the value of the optimal policy, both of which are close to zero. Small-cost bounds give the same rates as small-return bounds, however, they are more attractive as the cost of the initial policy can be high while the cost of the optimal policy can be low. Small-cost bounds for online learning were given by \citet{lee2020bias} 
for learning in tabular MDPs, while \citet{kakade2020information} derived such results
 for linear quadratic regulators (LQR).
To our knowledge we give the first small-cost bound for batch RL under completeness  \citep{antos-nips-2007,JMLR:v9:munos08a}. 

\paragraph{Small-cost bounds in bandits} Several works get small-cost bounds in contextual bandits \citep{allen2018make, foster2021efficient,olkhovskaya2023first}. 
In their Theorem~3,
\citet{foster2021efficient} show that for 
batch contextual bandits, playing the greedy policy with respect to a reward function estimated via log-loss enjoys a small-cost bound, whereas in their Theorem~2 they show that if a reward function is estimated
via squared loss, the greedy policy fails to achieve a small-cost guarantee. 
Our main result, \cref{thm:main-paper} can be viewed as the sequel to their Theorem 3. 
\citet{abeille2021instance} show that for stochastic logistic bandits, where the costs/rewards are Bernoulli, the regret can be made to scale with the variance of the optimal arm. This is simultaneously a small-cost and small-return bound. They achieve this bound by employing optimism together with maximum likelihood estimation (MLE). \citet{janz2023exploration} extend this result and show that this result continues to hold if the reward distribution comes from a ``self-concordant'', single-parameter family. 

\paragraph{Theory on batch RL}
The theory literature on batch RL has largely  focused on proving sample efficiency rates. \citet{pmlr-v97-chen19e} proved that \fqisq gets a rate optimal bound of $1/\sqrt{n}$ when realizability, concentrability and completeness hold. \citet{foster2021offline} then show that if one assumes concentrability then completeness is a necessary assumption for sample efficient batch RL. \citet{pmlr-v139-xie21d} prove that if one uses the stronger notation of concentrability from \citet{munos2003error} then sample efficient batch RL is possible even with only realizability. To our best knowledge, the previous theoretical works on batch RL have only considered algorithms where value estimation used  squared loss \citep[e.g.,][]{antos-nips-2007,farahmand2011regularization,Pires2012StatisticalLE,pmlr-v97-chen19e,pmlr-v139-xie21d}. 

\paragraph{Concurrent empirical work} 
The concurrent and independent empirical work of \citet{farebrother2024stop} also advocates for log-loss, but they end up with the approach that is used in distributional RL, which reduces regression to multiclass classification. 
This is not only more complicated than using log-loss (more parameters), but also introduces irreducible bias, whereas our approach avoids this. %Because of other factors, it remains unclear whether the log-loss based approach is better.

%related works

%conclusion

%conclusion
\section{Conclusions}

By proving that \fqilog is more sample efficient in MDPs with a small optimal cost $\bar{v}^\star$
than \fqisq, 
we showed that in batch RL the loss function genuinely matters. 
%\fqilog is the first efficient batch RL algorithm to get a small-cost bound provided an efficient regression oracle is available for log-loss. 
We believe our result holds generally and can be extended to any batch RL setting where the squared loss has bounded error, such as when pessimistic methods are used.
Another intriguing extension would be deriving small-cost bounds in batch RL with only realizability (and a stronger concentrability assumption), perhaps following the analysis of \citet{pmlr-v139-xie21d}.
\citet{pmlr-v162-wagenmaker22b} get small-return bounds for online RL in linear MDPs. 
Can we get small-cost bounds in online RL with linear function approximation? When and how? 
Finally, our mountain car experiments indicate 
that log-loss might perform well in goal-oriented MDPs \cite{bertsekas1995}, 
where the learner is tasked with reaching some goal and this is possible.
However, in this formulation there is no ``pressure'' for reaching the goal as quickly as possible.
It remains to be seen how this could be incorporated in algorithms like ours. 
(Staying away from failure zones for as long as possible can be easily formulated with the help of discounting.) Our experiments concerning whether switching to the log-loss is sufficient to replicate successes of the more complex C51 distributional RL algorithm were inconclusive. Here, it will be interesting to investigate whether switching to log-loss, but also keeping the losses used by C51 as auxiliary loss can lead to a method that outperforms both C51 and our method.

\section*{Acknowledgments}
The authors would like to thank Roshan Shariff for pointing out a bug in an earlier version of our proof. Csaba Szepesvári also gratefully acknowledges funding from the Canada CIFAR AI Chairs Program, Amii and NSERC.
\section*{Impact Statement}
This paper presents work whose goal is to advance the field of reinforcement learning theory. There are many potential societal consequences of our work.

\bibliography{ref}
\bibliographystyle{icml2024}

%%%%%%%%%%%%%%%%%%%%%%%%%%%%%%%%%%%%%%%%%%%%%%%%%%%%%%%%%%%%%%%%%%%%%%%%%%%%%%%
%%%%%%%%%%%%%%%%%%%%%%%%%%%%%%%%%%%%%%%%%%%%%%%%%%%%%%%%%%%%%%%%%%%%%%%%%%%%%%%
% APPENDIX
%%%%%%%%%%%%%%%%%%%%%%%%%%%%%%%%%%%%%%%%%%%%%%%%%%%%%%%%%%%%%%%%%%%%%%%%%%%%%%%
%%%%%%%%%%%%%%%%%%%%%%%%%%%%%%%%%%%%%%%%%%%%%%%%%%%%%%%%%%%%%%%%%%%%%%%%%%%%%%%
\newpage
\appendix
\onecolumn
\section{Preliminary results}

In this section we introduce and prove some elementary inequalities that connect useful metrics on function spaces, and then
state a concentration result of \citet{foster2021efficient} for the log-loss estimator.
The concentration result gives a high probability upper bound on the error of the log-loss estimator, as measured by the
integrated binary Hellinger loss (defined below). This result is central to our analysis.
The elementary inequalities connect the integrated
binary Hellinger loss to both the Hellinger distance and the triangular discrimination, and will reduce the analysis of our algorithm to studying the approximation error of its value function estimates $\{f_j\}_{j=1}^k$.

The analysis of \fqilog revolves around controlling the Hellinger distance, which is a distance between nonnegative
integrable functions. In particular, for $\lambda$-integrable functions $f, g \ge 0$, the Hellinger distance between $f$ and $g$ is defined as
\begin{align*}
\Hell(f,g) = \frac1{\sqrt 2} \norm{\sqrt{f}-\sqrt{g}}_{2,\lambda}\,.
\end{align*}

\subsection{Some basic inequalities}
\label{sec:basicineq}
Given real numbers $p, q \in [0, 1]$, we define the \emph{binary Hellinger loss} of $p$ and $q$ as
% \todoc{so maybe this should have been $\ell_{\Hell}(p,q)$.}
\begin{align}\label{defn:hell}
        \hell^2(p, q) = \frac{1}{2}\left(\sqrt{p}-\sqrt{q}\right)^2 + \frac{1}{2}\left(\sqrt{1-p} - \sqrt{1-q}\right)^2\,,
\end{align}
and immediately observe that $0 \le \hell^2(p,q) \le 1$.
Note that the \emph{Hellinger distance} between two distributions $P$ and $Q$ over a common domain is defined as $\frac{1}{\sqrt{2}} \norm{ \sqrt{p}-\sqrt{q} }_{2,\lambda}$, where $p = dP/d\lambda$ and $q = dQ/d\lambda$ are the densities of $P$ and $Q$ with respect to a dominating distribution $\lambda$. \todoc{citation}
Thus the binary Hellinger loss between $p$ and $q$, $\hell^2(p, q)$,
is the squared Hellinger distance between Bernoulli distributions with means $p$ and $q$.

\begin{lemma}\label{prop:hellinger-triangular-scalar}
For all $p,q \in [0,1]$, we have 
\begin{equation}
    \frac{1}{4} \frac{(p-q)^2}{p+q} \leq \frac{1}{2} \left(\sqrt{p}-\sqrt{q}\right)^2 
    \leq \hell^2(p, q)\,,
\end{equation}
where for $p=q=0$ we define the left-hand side to be zero.
\end{lemma}
\begin{proof}
  If $p=q=0$ then equality holds trivially, and otherwise $(\sqrt{p}+\sqrt{q})^2 \leq 2(p + q)$ implies
  \begin{align*}
    \MoveEqLeft
      \frac{(p-q)^2}{4(p+q)} \leq  \frac{(p-q)^2}{2(\sqrt{p}+\sqrt{q})^2} = \frac{1}{2} \left(\sqrt{p}-\sqrt{q}\right)^2
      \le \frac{1}{2} \left(\sqrt{p}-\sqrt{q}\right)^2 + \frac{1}{2}\left(\sqrt{1-p}-\sqrt{1-q}\right)^2 = \hell^2(p,q)\,. \qedhere
    \end{align*}
\end{proof}
The next result holds for an extended definition of $\hell^2$ 
that replaces the inputs $p, q \in [0,1]$ with functions $f,g: \mcX \to [0,1]$.
Given such functions, we define $\hell^2(f,g): \mcX \to [0,1]$ by 
\begin{align*}
(\hell^2(f,g))(x) = \hell^2( f(x) , g(x))\,, \qquad x\in \mcX\,.
\end{align*}
With this definition in hand, the following is a straightforward corollary of \cref{prop:hellinger-triangular-scalar}.
\begin{corollary}\label{cor:hellinger-triangular-scalar}
For any distribution $\nu$ over the set $\mcX$ and any measurable functions $f,g: \mcX \to [0,1]$, 
\begin{align*}
  \norm{\frac{f - g}{\sqrt{f + g}}}_{2,\nu} \le \sqrt{2} \norm{\sqrt{f} - \sqrt{g}}_{2,\nu} \le 2 \norm{ \hell^2(f,g) }_{1,\nu}^{1/2}\,,
\end{align*}
where for $f(x)=g(x)=0$ we define $\frac{f(x) - g(x)}{\sqrt{f(x) + g(x)}}=0$.
\end{corollary}
We call the quantity $\norm{\hell^2(f,g)}_{1,\nu}$, which appears on the right hand side above, the \emph{integrated binary Hellinger loss} between $f$ and $g$. \todos{this name will never be mentioned again}
Squaring all quantities and dividing through by 4 yields the equivalent inequalities
\begin{align*}
\frac{1}{4}\norm{\frac{f - g}{\sqrt{f + g}}}^2_{2,\nu} \leq \frac{1}{2} \norm{\sqrt{f} - \sqrt{g}}^2_{2,\nu} \le  \norm{ \hell^2(f,g) }_{1,\nu}\,.
\end{align*}

In words, the integrated binary Hellinger loss between $f$ and $g$ is lower bounded by their \emph{squared} Hellinger distance, which itself is lower bounded by one quarter of the triangular discrimination between them.
The latter is essentially the squared distance between $f$ and $g$, but rescaled pointwise by $1/\sqrt{f + g}$.
Because $|a-b|/\sqrt{a+b}\le \eps$ implies $|a-b|\le \eps \sqrt{a+b}$, we see that a bound on the rescaled distance between two values tightens the bound between them whenever $\sqrt{a+b}<1$.
Exploiting this property is key to our later analysis. 
\begin{proof}[Proof of \cref{cor:hellinger-triangular-scalar}]
Apply \cref{prop:hellinger-triangular-scalar} pointwise and then integrate both sides of the inequalities over $\nu$, before taking square roots and multiplying by 2 to simplify the constants.
\end{proof}

\subsection{Concentration for the log-loss estimator}
Fix a set $\cX$, which, for the sake of avoiding measurability issues, is assumed to be finite.
Let $(X_1,Y_1),\dots,(X_n,Y_n)$ be independent, identically distributed random elements taking values in $\cX\times [0,1]$.
Let $f^\star$ be the regression function underlying $\nu$: $f^\star(x)=\Eb{Y_1\,|\, X_1=x}$. \todos{should $\nu$ be defined before this, or is this meant to serve as definition?}
Let $\mcF \subseteq [0,1]^\mcX$ be a finite set of $[0,1]$-valued functions with domain $\cX$.
Recall the log-loss estimator: \todos{first time this is defined}
\begin{align*}
    \hat f_{\log}= \argmin_{f\in\Fcal} \sum_{i=1}^n \elllog(f(X_i);\,Y_i),
\end{align*}
where, for $y,y'\in [0,1]$, \todoc{what happens at the boundary!?} \todoa{Not sure this was the setting in Dylan's theorem?} \todoc{my question was rhetoric. I mean, the reader needs to be told what happens and you need to know.. Hint: we need to say that we let $0\log \infty = \lim_{x\to 0} x \log 1/x = 0$. I am just saying in general be careful so that whatever you write is well-defined.}
\todoc{note that we need $y\in [0,1]$, but strictly speaking, $y'$ could take any value. this allows the targets $Y_i$ to lie outside of the $[0,1]$ range..}
\begin{align*}    
     \elllog( y;\,y') = y'\log \frac{1}{y} + (1-y')\log\frac{1}{1- y}\,,
\end{align*}
where we define $0\log \infty = \lim_{x\to 0} x \log 1/x = 0$.
\citet{foster2021efficient} show the following concentration result for $\hat f_{\log}$, which we will need:
\begin{theorem}\label{thm:concentration}
Suppose $f^\star\in\mcF$. Let $D_n = \{(X_i,Y_i)\}_{i=1}^n$.
Then, for any $\delta\in(0,1)$, with probability at least $1-\delta$, we have
\begin{align*}
%    \mathbb{E}[ \hell^2(\hat f_{\log}(X)\|\, f^\star(X))|D_n] 
    \| \hell^2(\hat f_{\log},f^\star) \|_{1,\nu}
    \leq \frac{2\log(|\Fcal|/\delta)}{n}\,,
\end{align*}
where $\nu$ denotes the common distribution of $X_1,\dots,X_n$. \todos{does this use of $\nu$ conflict with the one just above? I thought if $f^\star$ was the regression function of $\nu$ then it should be the dist of $(X_i,Y_i)$?}
\end{theorem}

\begin{proof}
The result follows from the last equation on page 24 of the arXiv version of the paper by \citet{foster2021efficient} with $A = 1$. \qedhere
\end{proof}

\section{Proof of \cref{thm:main-paper}}\label{sec:appendix-proof}
In this section we give the main steps of the proof of \cref{thm:main-paper}. For the benefit of the reader, we first reproduce the text of the theorem.
\begin{theorem}\label{thm:main-appendix}
    Given a dataset $D_n = \{(S_i,A_i,C_i,S_{i}')\}_{i=1}^n$ with $n\in\NN$ and a finite function class $\mcF \subseteq [0,1]^{\mcS \times \mcA}$ that satisfy \cref{asp:data,asp:concentrability,asp:completeness}, it holds with probability $1-\delta$ that the output policy of \fqilog after $k$ iterations, $\pi_{k} = \pi_{f_k}$, satisfies 
    \begin{align*}
        \bar{v}^{\pi_{k}} - \bar{v}^\star  \leq \tilde C \left(\frac{1}{(1-\gamma)^2}\sqrt{\frac{\bar{v}^\star C \log\left(\lvert\mcF\rvert^2/\delta\right)}{n}} 
        + \frac{ C\log\left(\lvert\mcF\rvert^2/\delta\right)}{(1-\gamma)^4 n} 
        + \frac{\gamma^{\frac{k}{2}}}{1-\gamma} 
        + \frac{ \gamma^k}{(1-\gamma)^2}\right)\,.
    \end{align*}
    where $\tilde C  > 0 $ is an absolute constant. \todos{since we're changing the statement to have the full equations, I think we should also add in names of assumptions (reasonable distribution, concentrability with coefficient $C$, and completeness)}
\end{theorem}
The proof is reduced to two propositions and some extra calculations. 
We start by stating the two propositions first. The proofs of these propositions require more steps and will be developed in their own sections, following the proof of the main result, which ends this section.

The first proposition shows that the error of a policy that is greedy with respect to an action-value function $f:\mcS \times \mcA \to [0,\infty)$ 
can be bounded by the triangular discrimination between the action-value function and $q^\star$, the optimal action-value function in our MDP.
To state this proposition, for $f$ as above, we define $\Delta_f: \mcS \times \mcA \to [0,\infty)$, the pointwise triangular deviation of $f$ from $q^\star$: \todos{above it's called discrimination but now it's called deviation. Can we call it discrimination everywhere?}
\begin{align*}
    \Delta_f(s,a) = \frac{ f(s,a) - q^\star(s,a)}{\sqrt{f(s,a) + q^\star(s,a)}}\,,  \qquad \qquad (s,a)\in \mcS\times \mcA\,.
\end{align*}
To state the proposition, recall that for a distribution $\eta$ over the states and a stationary policy $\pi$, we let
$\eta\times \pi$ denote the joint probability distribution over the state-action pairs resulting from first sampling a state $S \sim \eta$ and then an action $A \sim \pi(S)$. \todoc{review notation/terminology for policies. also, this is just applying a Markov kernel.. Same as $\nu P$. So why the $\times$?}
\todos{we'll have to introduce this above if we're gonna say recall}
With this, the first proposition is as follows:
\begin{proposition}\label{lem:first-order-decomposition2}
Let $f: \mcS \times \mcA \rightarrow [0,\infty)$  and let $\pi = \pi_f$ be a policy that is greedy with respect to $f$. 
Define $D_f = \sup_{h\ge 1} \max(\norm{\Delta_f}_{2,\eta_h^{\pi}\times \pi},\norm{\Delta_f}_{2,\eta_h^{\pi}\times \pi^\star})$. 
    Then, it holds that
    \begin{align*}
%        \MoveEqLeft
        \bar v^{\pi}-\bar v^\star\le %\\
        %&
        \frac{22 D_f}{1-\gamma}\sqrt{2 \bar{v}^\star} + \frac{512D_f^2}{(1-\gamma)^2}\,. 
    \end{align*}
\end{proposition}
Recall that above $\eta_h^\pi$ is the distribution induced over the states in step $h$ when $\pi$ is followed from the start state distribution $\eta_1$.
As expected, the proof uses the performance difference lemma, followed by arguments that relate the stage-wise expected error that arises from the performance difference lemma to the ``size'' of $\Delta_f$.

When the above proposition is applied to $f = f_k$, the action-value function obtained in the $k^\text{th}$ iteration of our algorithm, we see that it remains to bound $D_{f_k}$.
The bound will be based on the second proposition:
\begin{proposition}\label{lem:contract-concentration}
For any admissible distribution $\nu$ over $\mcS \times \mcA$ that may also depend on the data $D_n$,\todoc{I commented out the assumptions.. Usually one says before the lemmas that the assumptions hold..}
for any $\delta \in (0, 1)$, $k\ge 1$,
%    Given a dataset $D = \left\{(s_i,a_i,r_i,s_{i+1})\right\}_{i=1}^n$ with $n \in \mathbb{N}$, \todoc{this needs work!}
%    a policy $\pi : S \rightarrow \mcM_1(A)$ 
%    a function class $\mcF$ under \cref{asp:concentrability,asp:realizability,asp:completeness,asp:bounded-targets} 
    with probability $1-\delta$, 
%    we have that after $k\ge 1$ iterations of FQI, \cref{alg:fqi}, 
%    the score function $f_k$ satisfies
    \begin{equation}
    	 \norm{ \Delta_{f_k} }_{2,\nu}
         \leq 
         \sqrt{\frac{32 C\log\left( \lvert\mcF\rvert^2/\delta\right)}{(1-\gamma)^2 n}} + \sqrt{2} \gamma^{\frac{k}{2}}\,,
    \end{equation}
    where $f_k$ denotes the value function computed by FQI, \cref{alg:fqi}, in step $k$ based on the data $D_n$.
\end{proposition}
The proof of this proposition uses 
{\em (i)} showing that $\mcT$ enjoys some contraction properties
with respect to appropriately chosen Hellinger distances; 
{\em (ii)} using these contraction properties to show that the Hellinger distance between $f_k$ and $q^\star$ is controlled by the Hellinger distances between $f_k$ and $\mcT f_k$,
and then using the results of the previous section to show that these are controlled by the algorithm.

With these two statements in place, the proof the main theorem is as follows:
\begin{proof}[Proof of \cref{thm:main-appendix}]
	Fix $k\ge 1$.
	For $h\ge 1$, 
    let $\eta_h^k = \eta_h^{\pi_k}$,
    $D_{f_k} = \sup_{h\ge 1}\max(\norm{\Delta_{f_k}}_{2,\eta_h^{k}\times \pi_k}, \norm{\Delta_{f_k}}_{2,\eta_h^{k}\times \pi^\star})$.
    Since, by definition, $\pi_k$ is greedy with respect to $f_k$,
    we can use \cref{lem:first-order-decomposition2} to get
    \begin{align*}
        \bar{v}^{\pi_{k}} - \bar{v}^\star 
        &\leq \frac{22\sqrt{2}D_{f_k}}{1-\gamma}\sqrt{\bar{v}^\star} + \frac{512D_{f_k}^2}{(1-\gamma)^2}\,. \numberthis \label{eq:bvb}
    \end{align*}
    
    It remains to bound $D_{f_k}$.
     An application of \cref{lem:contract-concentration} gives that 
	for any $0<\delta<1$, 
    with probability $1-\delta$, 
    \begin{equation}
    	 S:=\sup_{\nu \textrm{ admissible}}\norm{ D_{f_k} }_{2,\nu}
         \leq 
         \sqrt{\frac{32 C\log\left( \lvert\mcF\rvert^2/\delta\right)}{(1-\gamma)^2 n}} + \sqrt{2} \gamma^{\frac{k}{2}}\,.
    \end{equation}
	Since $\eta_h^{k}\times \pi_k$ and $\eta_h^{k}\times \pi^\star$ are admissible, as can be easily seen with an argument similar to that used in 
	the proof of \cref{cor:ep},
	it follows that with probability $1-\delta$,
	\begin{align}
		D_{f_k} \le  S \le  \sqrt{\frac{32 C\log\left( \lvert\mcF\rvert^2/\delta\right)}{(1-\gamma)^2 n}} +  \sqrt{2} \gamma^{\frac{k}{2}}\,. \label{eq:dfkb}
	\end{align}
	Squaring both sides and
   	using the inequality $(a+b)^2 \leq 2a^2 + 2b^2$, we get that the inequality
    \begin{align*}
	D_{f_k}^2 \le  \frac{ 64 C\log\left( \lvert\mcF\rvert^2/\delta\right)}{(1-\gamma)^2 n} + 4 \gamma^k
	\end{align*}
	also holds, on the same event when \cref{eq:dfkb} holds.
	Plugging these bounds into \cref{eq:bvb}, we get that with probability at least $1-\delta$,
    \begin{align*}
        \bar{v}^{\pi_{k}} - \bar{v}^\star &\leq \frac{22\sqrt{2}D_{f_k}}{1-\gamma}\sqrt{\bar{v}^\star} + \frac{512D_{f_k}^2}{(1-\gamma)^2} \\
        &\leq \frac{176}{(1-\gamma)^2}\sqrt{\frac{\bar{v}^\star C \log\left(\lvert\mcF\rvert^2/\delta\right)}{n}} 
        		+ \frac{32768 C\log\left(\lvert\mcF\rvert^2/\delta\right)}{(1-\gamma)^4 n} 
				+ \frac{44\gamma^{\frac{k}{2}}}{1-\gamma} 
				+ \frac{2048 \gamma^k}{(1-\gamma)^2}\,. \qedhere
    \end{align*}
\end{proof}

\subsection{An error bound for greedy policies: Proof of \cref{lem:first-order-decomposition2}}
The analysis in this section is inspired by the proof of Lemma~1 of \citet{foster2021efficient}. We deviate from their analysis to avoid introducing an extra $|\mcA|$ factor in the bounds. 

\paragraph{Additional Notations} For any function $g:\mcS \times \mcA \rightarrow \mbR$ and policy $\pi$, define $g(s,\pi) = \sum_{a \in \mcA} \pi(a|s) g(s,a)$.  For any $\nu \in \mcM_1(S\times A)$ which is an $|\mcS\times \mcA|$ dimensional row vector, define $\nu P \in \mcM_1(\mcS)$ as the distribution obtained over the states by first sampling a state-action pair from $\nu$ and then following $P$. That is, $\nu P$ is the distribution of $S'\sim P(\cdot|S,A)$ where $(S,A) \sim \nu$. \todoc{This should be $\nu P$. The standard convention is to think of distributions as row vectors. A probability kernel, such as $P$, is a matrix of appropriate size. Then $\nu P$ gives a row vector, with the usual vector-matrix multiplication rules, which is exactly what we want here.}
We can think of $\nu P$ as the distribution we get when $P$ is composed with $\nu$. For any function $f : \mcS \times \mcA \rightarrow [0,\infty)$,
%\todoc{the ``any $q^\star$'' with $q^\star$ hidden in the notation is weird.}
in addition to $\Delta_f$, we also define $\xi_f: \mcS \times \mcA \to \mbR$ as
\begin{equation}
    \xi_f(s,a) = f(s,a) + q^\star(s,a)\,, \qquad \qquad (s,a)\in \mcS\times \mcA\,.
\end{equation}
Recall that $\mcF$ contains $[0,1]$-valued functions with domain $\mcS \times \mcA$ and as such for any $f\in \mcF$, $\Delta_f$ and $\xi_f$ are well-defined.

We start with the performance difference lemma, which is stated without a proof:
\begin{lemma}[Performance Difference Lemma of \citeauthor{10.5555/645531.656005}]
\label{lem:perf-diff}
    For policies $\pi,\bar\pi : \mcS \rightarrow \mcM_1(\mcA)$, we have 
    \begin{equation}
        \bar v^\pi - \bar v^{\bar\pi} = \sum_{h=1}^\infty \gamma^{h-1}\innerproduct{\eta_h^\pi}{ q^{\bar\pi}(\cdot,\pi) - v^{\bar\pi}}\,. 
        %\frac{1}{1-\gamma}\EE_{s \sim \eta_\gamma^{\pi}}\EE_{a \sim \pi(\cdot|s)}\left[q^{\pi'}(s,a) - v^{\pi'}(s)\right]\,.
    \end{equation}
\end{lemma}
\begin{proof}
See Lemma 6.1 by \citet{10.5555/645531.656005}.
\end{proof}

The next lemma upper bounds the one-norm distance between a nonnegative-valued function $f:\mcS \times \mcA \to [0,\infty)$ and $q^\star$ in terms of appropriate norms of $\Delta_f$ and $\xi_f$.
\begin{lemma}\label{lem:triangular-discrim-bound}
    For any function $f: \mcS \times \mcA \to [0,\infty)$ 
    and distribution $\nu \in \mcM_1(S\times A)$, we have
    \begin{equation}
        \lVert f - q^\star \rVert_{1,\nu} \leq \lVert \xi_f \rVert_{1,\nu}^{1/2} \cdot\lVert \Delta_f\rVert_{2,\nu} \,.
    \end{equation}
\end{lemma}
\begin{proof}
  We have
 \begin{align}
      \lVert f - q^\star \rVert_{1,\nu} &= \norm{
      		\sqrt{f+q^\star}
			\cdot
      		\frac{f - q^\star}{\sqrt{f+q^\star}}
      }_{1,\nu} \\
     &\leq\norm{f+q^\star}_{1,\nu}^{1/2} \cdot \norm{\frac{(f-q^\star)^2}{f+q^\star}}_{1,\nu}^{1/2} \tag{Cauchy-Schwarz} \\
     &= \norm{\xi_f}_{1,\nu}^{1/2} \cdot \norm{\Delta_f}_{2,\nu}\,. \qedhere
 \end{align}
\end{proof}

\begin{lemma}\label{lem:regret_decomp}
    Let $f: S \times A \rightarrow [0,\infty)$  and let $\pi = \pi_f$ be a policy that is greedy with respect to $f$ and $h$ be a nonnegative integer.
    Then it holds that \todoc{why use $\tilde \pi$. why not simply use $\pi$? less clutter is good on the eyes.}
    \begin{align*}
        \langle \eta_h^{\pi}, \, q^\star(\cdot,\pi)-v^\star\rangle 
        \leq \left(\norm{ \xi_f }_{1,\eta_h^{\pi}\times \pi}^{1/2} + \norm{ \xi_f }_{1,\eta_h^{\pi}\times \pi^\star}^{1/2}\right)
        		\left( \norm{ \Delta_f}_{2,\eta_h^{\pi}\times \pi} + \norm{ \Delta_f}_{2,\eta_h^{\pi}\times \pi^\star} \right)\,.
	\end{align*}
\end{lemma}
\begin{proof} 
We have
\begin{align*}
         \innerproduct{ \eta_h^{\pi}} {q^\star(\cdot,\pi)-v^\star} 
        &= \innerproduct{ \eta_h^{\pi}}{q^\star(\cdot,\pi)-q^\star(\cdot,\pi^\star)} \tag{Defn of $v^\star$} \\
        &\leq \innerproduct{ \eta_h^{\pi}}{ q^\star(\cdot,\pi)-f(\cdot,\pi) + f(\cdot,\pi^\star) - q^\star(\cdot,\pi^\star)} \tag{$f(\cdot,\pi) \leq f(\cdot,\pi^\star)$ by defn of $\pi$} \\
        &\leq  \norm{q^\star-f}_{1,\eta_h^{\pi}\times \pi} + \norm{f-q^\star}_{1,\eta_h^{\pi}\times \pi^\star}\,. \tag{triangle inequality}
\end{align*} \todoc[inline]{(defn of $v^\star$ or properties of $v^\star/q^\star$. By the way, these should be mentioned upfront! Maybe in the notation section and we can refer back to there.}
Now, 
\begin{align*}
\MoveEqLeft
    \norm{q^\star-f}_{1,\eta_h^{\tilde\pi}\times \tilde\pi} + \norm{f-q^\star}_{1,\eta_h^{\tilde\pi}\times \pi^\star}\\  
    &\leq \norm{ \xi_f }_{1,\eta_h^{\tilde\pi}\times \tilde\pi}^{1/2} \cdot\norm{ \Delta_f}_{2,\eta_h^{\tilde\pi}\times \tilde\pi} + \norm{ \xi_f}_{1,\eta_h^{\tilde\pi}\times \pi^\star}^{1/2} \cdot\norm{ \Delta_f}_{2,\eta_h^{\tilde\pi}\times \pi^\star}\tag{\cref{lem:triangular-discrim-bound}}\\
    &\leq \left(\norm{ \xi_f }_{1,\eta_h^{\tilde\pi}\times \tilde\pi}^{1/2} + \norm{ \xi_f }_{1,\eta_h^{\tilde\pi}\times \pi^\star}^{1/2}\right)\left( \norm{ \Delta_f}_{2,\eta_h^{\tilde\pi}\times \tilde\pi} + \norm{ \Delta_f}_{2,\eta_h^{\tilde\pi}\times \pi^\star} \right)\,. \qedhere
\end{align*}
\end{proof}

%%%%%%%%%%%%%%%%%%%%%%%%%%%%%%%%%%%%%%%%%%%%%%%%

\begin{lemma}\label{lem:bound-sum-f-fstar}
    For any function 
    $f: \mcS \times \mcA \rightarrow [0,\infty)$
    and  distribution $\nu \in \mcM_1(\mcS\times \mcA)$, it holds that \todoc{where is admissibility used? I got rid of it. It was not used. In general, one should eliminate conditions not used.}
    \begin{equation}
        \norm{f + q^\star}_{1,\nu}  \leq 4\norm{q^\star}_{1,\nu} + \norm{\Delta_f}_{2,\nu}^2\,.
    \end{equation}
\end{lemma}
\begin{proof}
    Let $f \in \mcF$ be fixed, we have
    \begin{align*}
        \norm{f + q^\star}_{1,\nu}  &= \norm{f - q^\star + q^\star  + q^\star}_{1,\nu}  \\
        &\leq 2\norm{q^\star}_{1,\nu} + \norm{f-q^\star}_{1,\nu} \tag{triangle inequality} \\
        %&= 2\EE_\nu [q^\star(s,a)] + \EE_\nu[\lvert (f - q^\star)(s,a)\rvert] \tag{Definition of $q^\star(s,a)$} \\
        &= 2\norm{q^\star}_{1,\nu}  + \norm{\sqrt{f+q^\star} \,\frac{f-q^\star}{\sqrt{f+q^\star}}}_{1,\nu}\\
        &\leq 2\norm{q^\star}_{1,\nu} + \frac{1}{2}\norm{f+q^\star + \frac{(f-q^\star)^2}{f+q^\star}}_{1,\nu} \tag{$a b\le \tfrac{a^2+b^2}{2}$ for $a,b$ nonnegative reals} \\
        &\leq 2\norm{q^\star}_{1,\nu} + \frac{1}{2}\norm{f+q^\star}_{1,\nu} + \frac{1}{2} \norm{\Delta_f}_{2,\nu}^2\,. \tag{triangle inequality}
    \end{align*}
    Rearranging and multiplying through by two gives the statement. \qedhere
\end{proof}
\begin{lemma}\label{lem:first-order-decomposition}
Let $f: \mcS \times \mcA \rightarrow [0,\infty)$  and let $\pi = \pi_f$ be a policy that is greedy with respect to $f$. 
Define $D_f = \sup_{h\ge 1} \max(\norm{\Delta_f}_{2,\eta_h^{\pi}\times \pi},\norm{\Delta_f}_{2,\eta_h^{\pi}\times \pi^\star})$. 
    Then, it holds that
    \begin{align*}
%        \MoveEqLeft
        \bar v^{\pi}-\bar v^\star%\\
        %&
        &\le 11D_f \sum_{h=1}^\infty \gamma^{h-1} \norm{v^\star}_{1,\eta_h^{k}}^{1/2} + \frac{28 D_f^2}{1-\gamma}\,.
    \end{align*}
\end{lemma}
\begin{proof}

    Recall that by the performance difference lemma, \cref{lem:perf-diff}, it holds that
    \begin{equation}
        \bar{v}^{\pi}-\bar{v}^\star = \sum_{h=1}^\infty \gamma^{h-1} \langle \eta_h^{\pi}, \, q^\star(\cdot,\pi)-v^\star\rangle\,.
        \label{eq:pdla}
    \end{equation}

    For the remainder of this proof we fix $h\ge 1$. For the $h^{\mathrm{th}}$ term from the above display, we have
    \begin{align*}
        \MoveEqLeft
        \langle \eta_h^{\pi}, \, q^\star(\cdot,\pi)-v^\star\rangle 
        \leq \left(\norm{ \xi_f }_{1,\eta_h^{\pi}\times \pi}^{1/2} + \norm{ \xi_f }_{1,\eta_h^{\pi}\times \pi^\star}^{1/2}\right)\left( \norm{ \Delta_f}_{2,\eta_h^{\pi}\times \pi} + \norm{ \Delta_f}_{2,\eta_h^{\pi}\times \pi^\star} \right) \tag{\cref{lem:regret_decomp}}\\
        &\leq \left(\sqrt{4\norm{q^\star}_{1,\eta_h^{\pi}\times \pi} + \norm{\Delta_f}_{2,\eta_h^{\pi}\times \pi}^2} +  \sqrt{4\norm{q^\star}_{1,\eta_h^{\pi}\times \pi^\star} + \norm{\Delta_f}_{2,\eta_h^{\pi}\times \pi^\star}^2}\right)\left(\norm{\Delta_f}_{2,\eta_h^{\pi}\times \pi} + \norm{\Delta_f}_{2,\eta_h^{\pi}\times \pi^\star}\right) \,.
        \tag{\cref{lem:bound-sum-f-fstar}}
    \end{align*}
 	Now recall that by definition
    \[
    \max\left\{\norm{\Delta_f}_{2,\eta_h^{\pi}\times \pi},\norm{\Delta_f}_{2,\eta_h^{\pi}\times \pi^\star}\right\} \le D_f\,.
    \] 
	Hence,
    \begin{align}\label{eqnimplicity-ineq}
    \MoveEqLeft
       \langle \eta_h^{\pi}, \, q^\star(\cdot,\pi)-v^\star\rangle\notag\\
       &\leq 2D_f\left(\sqrt{4\norm{q^\star}_{1,\eta_h^{\pi}\times \pi}+D_f^2} +  \sqrt{4\norm{q^\star}_{1,\eta_h^{\pi}\times \pi^\star}+D_f^2}\right) \\
        &\leq 2D_f \left(2D_f+\sqrt{4\norm{q^\star}_{1,\eta_h^{\pi}\times \pi}} +  \sqrt{4\norm{q^\star}_{1,\eta_h^{\pi}\times \pi^\star}}\right)\tag{$\sqrt{a+b}\leq\sqrt{a}+\sqrt{b}$} \\
        &= 4 D_f^2 +4D_f\norm{q^\star}_{1,\eta_h^{\pi}\times \pi}^{1/2}+   4D_f\norm{q^\star}_{1,\eta_h^{\pi}\times \pi^\star}^{1/2}\notag \\
        &\leq 20D_f^2 + \frac{\norm{q^\star}_{1,\eta_h^{\pi}\times \pi} + \norm{q^\star}_{1,\eta_h^{\pi}\times \pi^\star}}{2}\,. \tag{$a b\le \tfrac{a^2+b^2}{2}$ for $a,b$ nonnegative reals, twice with $a=4D_f$} 
    \end{align}
    Using that $v^\star$, $q^\star$ are nonnegative valued and that $v^\star(\cdot) = q^\star(\cdot,\pi^\star)$, we calculate
    $\langle  \eta_h^{\pi},v^\star\rangle = \norm{v^\star}_{1,\eta_h^{\pi}} = \norm{q^\star}_{1, \eta_h^{\pi}\times\pi^\star}$.
    Thus, by the previous display, after rearranging, we get
    \begin{align*}
	    \norm{q^\star}_{1,\eta_h^{\pi}\times\pi} = \langle \eta_h^{\pi},\, q^\star(\cdot, \pi) \rangle \leq 40D_f^2 + 3\norm{q^\star}_{1,\eta_h^{\pi}\times\pi^\star}\,,
\end{align*}
 	where the equality used  the non-negativity of $q^\star$. 
	Plugging this back into the inequality in \cref{eqnimplicity-ineq} gives
    \begin{align*}
       \MoveEqLeft \langle \eta_h^{\pi}, \, q^\star(\cdot,\pi)-v^\star\rangle
        \leq  2D_f\left(\sqrt{4\norm{q^\star}_{1,\eta_h^{\pi}\times \pi^\star}+D_f^2} +  \sqrt{4\norm{q^\star}_{1,\eta_h^{\pi}\times \pi}+D_f^2}\right) \tag{restating \cref{eqnimplicity-ineq}}\\
        &\leq 2D_f\left(\sqrt{4\norm{q^\star}_{1,\eta_h^{\pi}\times \pi^\star}+D_f^2} +  \sqrt{160D_f^2 + 12\norm{q^\star}_{1,\eta_h^{\pi}\times\pi^\star}+D_f^2}\right) \\
        &\leq 2D_f\left(2\norm{q^\star}_{1,\eta_h^{\pi}\times \pi^\star}^{1/2}+D_f +  \sqrt{161}D_f + \sqrt{12}\norm{q^\star}_{1,\eta_h^{\pi}\times\pi^\star}^{1/2}\right) \tag{$\sqrt{a+b}\leq\sqrt{a}+\sqrt{b}$}\\
        &\leq 11D_f \norm{q^\star}_{1,\eta_h^{\pi}\times \pi^\star}^{1/2} + 28D_f^2 \,.
    \end{align*}
    Combining this with \cref{eq:pdla} gives the desired inequality. 

\end{proof}

\begin{lemma}\label{lem:non-negativity-vpi}
    For any policy $\pi : \mcS \rightarrow \mcM_1(\mcA)$ we have
    \begin{equation}
        \sum_{h=1}^\infty \gamma^{h-1}\,\sqrt{\innerproduct{\eta_h^\pi}{v^\pi}}
         \leq 
         \frac{2\sqrt{\bar{v}^\pi} }{1-\gamma}\,.
    \end{equation}
\end{lemma}
\begin{proof}
    Notice that \todoc{make sure $\innerproduct{\cdot}{\cdot}$ is defined.}
    \todoa{It is defined in the main.}
    \begin{align*}
        \bar{v}^\pi &= \innerproduct{\eta_1}{v^\pi} \\ 
        %\EE\left[\sum_{t=1}^\infty \gamma^{t-1}C_t\, \big|\, s\sim\eta_1,\pi\right] \\
        &= \innerproduct{\eta_1^\pi}{c(\cdot,\pi)} + \gamma\innerproduct{\eta_2^\pi}{c(\cdot,\pi)} + \gamma^2\innerproduct{\eta_3^\pi}{c(\cdot,\pi)} + \dotsb + \gamma^{h-1}\innerproduct{\eta_h^\pi}{v^\pi} \\
        &\geq \gamma^{h-1}\innerproduct{\eta_{h}^\pi}{v^\pi}
    \end{align*}
    where the inequality follows from the non-negativity of the costs. Simple rearrangement gives
    \begin{equation*}
        \innerproduct{\eta_{h}^\pi}{v^\pi} \leq \frac{\bar{v}^\pi}{\gamma^{h-1}}\,.
    \end{equation*}
    Using this inequality, we get
    \begin{align*}
        \sum_{h=1}^\infty \gamma^{h-1}\,\sqrt{\innerproduct{\eta_{h}^\pi}{v^\pi} } 
        &\leq \sum_{h=1}^\infty \gamma^{h-1}\,\sqrt{\frac{\bar{v}^\pi}{\gamma^{h-1}}} 
        = \sum_{h=1}^\infty \sqrt{\gamma^{h-1} \bar{v}^\pi}
        \le \frac{2\sqrt{\bar{v}^\pi} }{1-\gamma}\,,
        %\qquad\qquad\qquad\qquad
    \end{align*}
    where for the last inequality we used that $1/(1-\sqrt{\gamma}) \leq 2/(1-\gamma)$.
\end{proof}

With this we are ready to prove \cref{lem:first-order-decomposition2}:
\begin{proof}[Proof of \cref{lem:first-order-decomposition2}]
	For $h\ge 1$, 
    let $\eta_h = \eta_h^{\pi}$.
    Starting from 
    \cref{lem:first-order-decomposition}, we bound $\bar{v}^{\pi} - \bar{v}^\star$ as follows:
    \begin{align*}
         \bar{v}^{\pi} - \bar{v}^\star
        &\le 11D_f \sum_{h=1}^\infty \gamma^{h-1} \norm{v^\star}_{1,\eta_h}^{1/2} + 28 D_f^2 \sum_{h=1}^{\infty} \gamma^{h-1} \tag{\cref{lem:first-order-decomposition}}\\
        &= 11D_f \sum_{h=1}^\infty \gamma^{h-1} \sqrt{\innerproduct{\eta_h}{v^\star}} + \frac{28 D_f^2}{1-\gamma} \\
        &\leq 11D_f \sum_{h=1}^\infty \gamma^{h-1} \sqrt{\innerproduct{\eta_h}{v^{\pi}}} + \frac{28 D_f^2}{1-\gamma} \tag{Defn of $v^\star$} \\
        &\leq \frac{22D_f\sqrt{\bar v^{\pi}}}{1-\gamma} + \frac{28 D_f^2}{1-\gamma} \tag{\cref{lem:non-negativity-vpi}, $\qquad(\star)$}\\
        &\leq \frac{22^2D_f^2}{2(1-\gamma)^2} + \frac{ \bar v^{\pi} }{2} + \frac{28 D_f^2}{1-\gamma}\,.  \tag{$ab\le (a^2+b^2)/2$, $a,b\ge 0$}
    \end{align*}
    % where the second inequality uses \cref{lem:triangular-discrim-bound} and the third inequality uses . 
    Rearranging the last inequality obtained, we get
    \begin{align*}
    \bar{v}^{\pi} \le 2 \bar{v}^\star+
    \frac{22^2D_f^2}{(1-\gamma)^2} + \frac{56 D_f^2}{1-\gamma}\,.
	\end{align*}
    Plugging this bound into $(\star)$, we get
    \begin{align*}
        \bar{v}^{\pi} - \bar{v}^\star 
        &\leq \frac{22D_f}{1-\gamma}\sqrt{2\bar{v}^\star + \frac{22^2D_f^2}{(1-\gamma)^2} + \frac{56 D_f^2}{1-\gamma}} +  \frac{28 D_f^2}{1-\gamma} \\
        &\leq \frac{22\sqrt{2}D_f}{1-\gamma}\sqrt{\bar{v}^\star} + \frac{22^2D_f^2}{(1-\gamma)^2} + \frac{165D_f^2}{(1-\gamma)^{3/2}} + \frac{28D_f^2}{1-\gamma}
        	 \tag{$\sqrt{a+b}\le \sqrt{a}+\sqrt{b}$, $a,b\ge 0$ twice}\\
        &\leq \frac{22\sqrt{2}D_f}{1-\gamma}\sqrt{\bar{v}^\star} + \frac{512D_f^2}{(1-\gamma)^2}\,. \qedhere
    \end{align*}
\end{proof}

\subsection{Bounding the triangular deviation between $f_k$ and $q^\star$: Proof of \cref{lem:contract-concentration}}
As explained earlier, the analysis in this section uses contraction arguments that have a long history in the analysis of dynamic programming algorithms in the context of MDPs. 
The novelty is that we need to bound the triangular deviation to $q^\star$. As this has been shown to be upper bounded by the Hellinger distance (\cref{cor:hellinger-triangular-scalar}), we switch to Hellinger distances and establishes contraction properties of $\mcT$ with respect to such distances. This required new proofs.  
The change of measure arguments used in the ``error propagation analysis'' are standard.

The following lemma, at a high level, establishes that the map $f \mapsto f^\wedge$ (``$\min$-operator'') is a non-expansion over the set of nonnegative functions with domain $\mcS\times \mcA$ and $\mcS$, respectively, when these function spaces are equipped with appropriate norms:
\begin{lemma}\label{lem:non-expansion}
Define the policy $\pi_{f,g}(s) = \arg \min_{a \in A} \min\{f(s,a),g(s,a)\}$ and assume that $f,g:   \mcS \times \mcA \rightarrow [0,\infty)$. \todoc{what happened with $\mcS$ and $\mcA$ vs. $S$ and $A$?}\todoa{It should always be $\mcS$ and $\mcA$ my} Then, for any distribution $\eta \in \mcM_1(\mcS)$, we have that
\begin{equation}
    \norm{\sqrt{f^\wedge} - \sqrt{g^\wedge}}_{2,\eta} \leq \norm{\sqrt{f} - \sqrt{g}}_{2,\eta \times \pi_{f,f'}}\,.
\end{equation}
\end{lemma}
\begin{proof}
Notice that for two finite sets of reals,
$U = \{u_1,\dots,u_n\}$, $V = \{ v_1,\dots,v_m\}$, with $u_1 = \min U$, $v_j = \min V$, $u_1\le v_j$, $j\in [m]$,
we have
$|\min U - \min V|= v_j-u_1 \le v_1-u_1 \le |u_1-v_1|$. 
By taking the square root of all the elements in both $U$ and $V$, assuming these are nonnegative, we also get that 
\begin{align}
| \sqrt{\min U} - \sqrt{\min V}| \le \sqrt{v_1}-\sqrt{u_1}\le |\sqrt{u_1}-\sqrt{v_1}|\,. \label{eq:bmin}
\end{align}
Hence,
    \begin{align*}
	\norm{\sqrt{f^\wedge} - \sqrt{g^\wedge}}_{2,\eta}^2 
        &= \sum_{s \in S}\eta(s) \left(\sqrt{\min_{a \in A}f(s,a)} - \sqrt{\min_{a \in A}g(s,a)}\right)^2 \\
        &= \sum_{s \in S}\eta(s) \left(\sqrt{f\left(s,\pi_f(s)\right)} - \sqrt{g\left(s,\pi_{g}(s)\right)}\right)^2\\
        &\leq \sum_{s \in S}\eta(s) \left(\sqrt{f(s,\pi_{f,g}(s))} - \sqrt{g(s,\pi_{f,g}(s))}\right)^2  \tag{by \cref{eq:bmin}}\\
        &= \sum_{s \in S}\eta(s)\sum_{a \in A} \mathbb{I}\left\{a=\pi_{f,g}(s)\right\} \left(\sqrt{f(s,a)} - \sqrt{g(s,a)}\right)^2\\
        &= \norm{\sqrt{f} - \sqrt{g}}_{2,\eta \times \pi_{f,g}}^2
    \end{align*}
    where the inequality used the definition of $\pi_{f,g}(s)$. \qedhere
\end{proof}

We need two more auxiliary lemmas before we can show the desired contraction result for $\mcT$.
The first is an elementary result that shows that for $x\ge 0$, over the nonnegative reals the map $u \mapsto \sqrt{x+u}$ is a nonexpansion: 
\begin{lemma}\label{prop:sqrt-bounds}
    For any $x,a,b \geq 0$, we have
    \begin{equation}
        \big\lvert\sqrt{x+a} -\sqrt{x+b}\big\rvert \leq \big\lvert\sqrt{a}-\sqrt{b}\big\rvert\,. 
    \end{equation}
\end{lemma}
\begin{proof}
For $x\ge 0$, let $f(x)=\big\lvert\sqrt{x+a} -\sqrt{x+b}\big\rvert$.
Note that the desired inequality is equivalent to that for any $x\ge 0$, $f(x)\le f(0)$.
This, it suffices to show that $f$ is a decreasing function over its domain.

    Without loss of generality we may assume that $a>b$ (when $a=b$, the inequality trivially holds, and if $a<b$, just relabel $a$ to $b$ and $b$ to $a$).
    Hence, $f(x) = \sqrt{x+a} -\sqrt{x+b}$ for \emph{any} $x\ge 0$ by the monotonicity of the square root function.
    For $x>0$, $f$ is differentiable. Here, 
	we get
    \begin{align}
        f'(x)
        =
        \frac{\partial}{\partial x} \left(\sqrt{x+a} -\sqrt{x+b}\right)
        = -\frac{\left(\sqrt{x+a}-\sqrt{x+b}\right)^2}{2\sqrt{x+a}\sqrt{x+b}\left(\sqrt{x+a}-\sqrt{x+b}\right)}\leq 0\,.
    \end{align}
    Now, since $f$ is continuous over its domain, by the mean-value theorem, $f$ is decreasing over $[0,\infty)$. \qedhere
\end{proof}

The next result shows that for any probability distribution $\lambda$ over some set $\mcX$, the map $g \mapsto \sqrt{\ip{\lambda}{g}}$ is a nonexpansion from $H^2(\mcX,\lambda)$ to the reals,
where $H^2(\mcX,\lambda)$ is the space of nonnegative valued functions over $\mcX$ equipped with the Hellinger distance $d(g,h):= \| g^{1/2}-h^{1/2} \|_{2,\lambda}$.
\begin{lemma}\label{prop:cauchy-schwarz}
    Given a random element $X$ taking values in $\mcX$ and nonnegative-valued functions $g,g' : \mcX \rightarrow [0,\infty)$ such that $g(X)$ and $g'(X)$ are integrable, we have
    \begin{equation}
        \left(\sqrt{\EE g(X)} - \sqrt{\EE g'(X)} \right)^2 \leq \EE\left(\sqrt{g(X)}-\sqrt{g'(X)}\right)^2 < \infty\,.
    \end{equation}
\end{lemma}
\begin{proof}
The result follows by some calculation:
    \begin{align*}
        \left(\sqrt{\EE g(X)} - \sqrt{\EE g'(X)} \right)^2 &= \EE g(X) - 2 \sqrt{\EE g(X)}\sqrt{\EE g'(X)} + \EE g'(X) \\
        &\leq \EE g(X) - 2\EE\sqrt{g(X)g'(X)} + \EE g'(X) \tag{Cauchy-Schwarz}\\
        &= \EE \left[g(X) - 2\sqrt{g(X)g'(X)} + g'(X)\right] \\
        &= \EE \left(\sqrt{g(X)}-\sqrt{g'(X)}\right)^2\,,
    \end{align*}
    where the Cauchy-Schwarz step uses that $g$ and $g'$ are nonnegative.
    Finally, that 
    $\EE\left(\sqrt{g(X)}-\sqrt{g'(X)}\right)^2 < \infty$ follows because,
    from $(a+b)^2 \le 2(a^2 +b^2)$, and hence since $g(X)$ and $g'(X)$ are assumed to be integrable,
    $\EE\left(\sqrt{g(X)}-\sqrt{g'(X)}\right)^2\le 2 \EE{g(X)} + 2 \EE{g'(X)}<\infty$
\end{proof}

With this we are ready to prove that
the Bellman optimality operator is a contraction when used over nonnegative functions equipped with Hellinger distances defined with respect to appropriate measures:

\begin{lemma}\label{lem:contraction-fixed}
     For any  distribution $\nu \in \mcM_1(\mcS\times \mcA)$, and functions $f,g : \mcS \times \mcA \rightarrow [0,\infty)$ we have
    \begin{equation*}
        \norm{\sqrt{\mcT f} - \sqrt{\mcT g}}_{2,\nu}
       \leq \gamma^{1/2} \norm{\sqrt{f}-\sqrt{g}}_{2,\nu P\times\pi_{f,g}}\,.
    \end{equation*}
    % and 
    % \begin{equation}
    %     \EE_{(s,a)\sim\nu} \left(\sqrt{f(s,a)} - \sqrt{Q^\star(s,a)}\right)^2 \leq \frac{C}{1-\gamma}\, \EE_{(s,a) \sim \mu} \left(\sqrt{f(s,a)} - \sqrt{\mcT f(s,a)}\right)^2\,.
    % \end{equation}
\end{lemma}
\begin{proof}
    By the definition of $\mcT$,  we have $\mcT f = c + \gamma P f^\wedge$.
    Here, $P$ is viewed as an $SA \times S$ matrix, while $c$ and $f^\wedge$ are viewed as $S$-dimensional vectors where $S$ and $A$ denote the cardinalities of $\mcS$ and $\mcA$ respectively.
    Also recalling that we use $\sqrt{f}$ to denote the elementwise square root of $f$ for $f$ a vector/function, we have
     \todoc{Note that I like to use $S$, $A$ for cardinalities; keeping $\mcS$ and $\mcA$ for the sets. This requires work.}

%By the definition of $\mcT$,  we have $\mcT f(s,a) = c(s,a) + \gamma \innerproduct{P(\cdot|s,a)}{f^\wedge}$.
%Here, $P$ is viewed as an $SA \times S$ matrix, 
\todoc{not anymore it seems! change back.. let's discuss..?
I wrote $\mcT f = c + \gamma P f^\wedge$ on purpose.. With the above, you also need to say, operator application takes higher precedence than function evaluation.
Which is of course how things must be, but still, need to mention these conventions. And from the above the vector/matrix notation is gone. But that view is useful below.
Only change things you are absolutely sure about.. I don't want to go over things again (sorry, just trying to save time).
}
%\todoa{No worries, i changed it back}
%while $c$ and $f^\wedge$ are viewed as $S$-dimensional vectors.
%Also recalling that we use $\sqrt{f}$ to denote the elementwise square root of $f$ when $f$ a vector/function, we have
 %\todoc{Note that I like to use $S$, $A$ for cardinalities; keeping $\mcS$ and $\mcA$ for the sets. This requires work.}
    \begin{align*}
        \norm{\sqrt{\mcT f} - \sqrt{\mcT g}}_{2,\nu}^2 
        &= \norm{\sqrt{c + \gamma\, P f^\wedge} - \sqrt{c + \gamma\, P g^\wedge}}_{2,\nu}^2 \\
        &\leq \norm{\sqrt{ \gamma\, P f^\wedge} - \sqrt{ \gamma\,P g^\wedge}}_{2,\nu}^2\tag{\cref{prop:sqrt-bounds} and the defn. of $\norm{\cdot}_{2,\nu}$} \\
        &= \gamma\,\norm{\sqrt{ P f^\wedge } - \sqrt{  P g^\wedge }}_{2,\nu}^2  \\
        &\leq \gamma\,\norm{\sqrt{f^\wedge} - \sqrt{g^\wedge}}_{2,\nu P}^2 \tag{\cref{prop:cauchy-schwarz}} \\
        &\leq \gamma\, \norm{\sqrt{f} - \sqrt{g}}_{2,\nu P\times\pi_{f,g}}^2\,,\tag{\cref{lem:non-expansion}}
    \end{align*} 
    thus finishing the proof.
\end{proof}
The next result is a simple change-of-measure argument:
\begin{lemma}\label{lem:norm-concen}
    Let $\mu,\nu$ be any distributions over $\mcS \times \mcA$ and assume that $\nu$ is admissible. 
    Then for $p\geq1$ we have $\lVert \cdot \rVert_{p,\nu} \leq C^{1/p} \lVert \cdot \rVert_{p,\mu}$.
\end{lemma}
\begin{proof}
    For any function $g: \mcS \times \mcA \rightarrow \mbR$, we have
    \begin{align*}
        \lVert g \rVert_{p,\nu} &= \left(\sum_{(s,a) \in S\times A}\lvert g(s,a)\rvert^p \nu(s,a)\right)^{1/p} \\
        &\leq \left(\sum_{(s,a) \in S\times A}\lvert g(s,a)\rvert^p C\mu(s,a)\right)^{1/p} \tag{\cref{asp:concentrability}} \\
        &= C^{1/p} \left(\sum_{(s,a) \in S\times A}\lvert g(s,a)\rvert^p \mu(s,a)\right)^{1/p} = C^{1/p} \lVert g \rVert_{p,\mu}\,.
    \end{align*} \qedhere
\end{proof}

\begin{lemma}\label{lem:contraction}
    Let $\mu,\nu$ be any distributions over $\mcS \times \mcA$ and assume that $\nu$ is admissible.
    Then, for any     
     $f,f' : S \times A \rightarrow [0,\infty)$ we have
    \begin{align*}
        \norm{\sqrt{f} - \sqrt{q^\star}}_{2,\nu} \leq \sqrt{C}\norm{\sqrt{f} - \sqrt{\mcT f'}}_{2,\mu} + \sqrt{\gamma}\norm{\sqrt{f'} - \sqrt{q^\star}}_{2,\nu P\times\pi_{f',q^\star}}\,.
    \end{align*}
    and $\norm{\sqrt{f} - \sqrt{q^\star}}_{2,\nu}\leq \frac{\sqrt{C}}{1-\sqrt{\gamma}} \norm{\sqrt{f} - \sqrt{q^\star}}_{2,\mu}$.
    % and 
    % \begin{equation}
    %     \EE_{(s,a)\sim\nu} \left(\sqrt{f(s,a)} - \sqrt{q^\star(s,a)}\right)^2 \leq \frac{C}{1-\gamma}\, \EE_{(s,a) \sim \mu} \left(\sqrt{f(s,a)} - \sqrt{\mcT f(s,a)}\right)^2\,.
    % \end{equation}
\end{lemma}

\begin{proof} 
    We have
    \begin{align*}
         \norm{\sqrt{f} - \sqrt{q^\star}}_{2,\nu} &= \norm{\sqrt{f} - \sqrt{\mcT f'} + \sqrt{\mcT f'} -  \sqrt{\mcT q^\star}}_{2,\nu} \tag{$q^\star = \mcT q^\star$} \\
        &\leq \norm{\sqrt{f} - \sqrt{\mcT f'}}_{2,\nu} + \norm{\sqrt{\mcT f'} -  \sqrt{\mcT q^\star}}_{2,\nu} \tag{triangle inequality} \\
        %&\leq 2C\, \EE_{\mu} \left(\sqrt{f(s,a)} - \sqrt{\mcT f'(s,a)}\right)^2 + 2\gamma\EE_{s'\sim P(\nu)}\left(\sqrt{f'_V(s')} - \sqrt{V^\star(s')}\right)^2 \tag{\cref{asp:concentrability}, $*$}\\
        &\leq \sqrt{C}\norm{\sqrt{f} - \sqrt{\mcT f'}}_{2,\mu} + \sqrt{\gamma}\norm{\sqrt{f'} - \sqrt{q^\star}}_{2,\nu P\times\pi_{f',q^\star}}\,
    \end{align*}
    where the last inequality uses \cref{lem:norm-concen,lem:contraction-fixed}. For the second term let $f'=f$ and $\nu_0 = \argmax_\nu \norm{\sqrt{f} - \sqrt{q^\star}}_{2,\nu}$, then
    \begin{align*}
        \norm{\sqrt{f} - \sqrt{q^\star}}_{2,\nu_0} &\leq \sqrt{C} \norm{\sqrt{f} - \sqrt{q^\star}}_{2,\mu} + \gamma^{1/2} \norm{\sqrt{f} - \sqrt{q^\star}}_{2,\nu_0 P\times \pi_{f,q^\star}} \\
        &\leq\sqrt{C} \norm{\sqrt{f} - \sqrt{q^\star}}_{2,\mu} + \gamma^{1/2} \norm{\sqrt{f} - \sqrt{q^\star}}_{2,\nu_0 }\,.
    \end{align*}
    Therefore, $\norm{\sqrt{f} - \sqrt{q^\star}}_{2,\nu} \leq \norm{\sqrt{f} - \sqrt{q^\star}}_{2,\nu_0} \leq \frac{\sqrt{C}}{1-\sqrt{\gamma}} \norm{\sqrt{f} - \sqrt{q^\star}}_{2,\mu}$\,.
     \qedhere
\end{proof}

\begin{lemma}[Error propagation]\label{cor:ep}
Fix $k\ge 1$ and let $f_0,f_1,\dots,f_k: \mcS \times \mcA \to [0,\infty)$ be arbitrary functions such that $f_0$ takes values in $[0,1]$, 
$\nu,\mu$ distributions over $\mcS \times \mcA$ and assume that $\nu$ is an admissible distribution.
Then,
\begin{align*}
\norm{\sqrt{f_k} - \sqrt{q^\star}}_{2,\nu} \le
\gamma^{\frac{k}{2}}+
\frac{2\sqrt{C}}{1-\gamma} \max_{1\le \tau\le k} \norm{ \sqrt{f_\tau}-\sqrt{\mcT f_{\tau-1}}}_{2,\mu}\,.
\end{align*}
\end{lemma}
\begin{proof}
Define $(\nu_{i})_{0\le i \le k }$ via
$\nu_{k} =\nu$ and for $0\le i<k$, let $\nu_i = (\nu_{i+1} P) \times \pi_{f_i,q^\star}$.
Note that by assumption, $\nu_k$ is admissible. It then follows that $\nu_i$ for $0\le i<k$ is also admissible.
Indeed, if for some $0\le i<k$, $\pi=(\pi_0,\pi_1,\dots)$ is the nonstationary policy that realizes $\nu_{i+1}$ in step $s\ge 0$, $\pi' = (\pi_0,\pi_1,\dots,\pi_s,\pi_{f_i,q^\star}, \pi_{s+1},\dots)$ is a policy that realizes
$\nu_i$ in step $s+1$.

Hence,
    \begin{align*}
        \norm{\sqrt{f_k} - \sqrt{q^\star}}_{2,\nu}
        &= \norm{\sqrt{f_k} - \sqrt{q^\star}}_{2,\nu_k}\tag{definition of $\nu_k$}\\
        &\leq \sqrt{C} \, \norm{\sqrt{f_k} - \sqrt{\mcT f_{k-1}}}_{2,\mu} + \sqrt{\gamma} \, \norm{\sqrt{f_{k-1}} - \sqrt{q^\star}}_{2,\nu_{k-1}}\,,
        \tag{\cref{lem:contraction}}
    \end{align*}
    where the second inequality uses \cref{lem:contraction}
    while setting 
    		$f,f',\nu,\mu$ 
				to 
    		$f_k,f_{k-1},\nu_k$ and $\mu$ (the data generating distribution), 
			respectively,
   	and noting that, by definition, 
	$\nu P\times\pi_{f',q^\star}$ 
	of the \namecref{lem:contraction} is 
    $(\nu_k P) \times \pi_{f_{k-1},q^\star}=\nu_{k-1}$, and that, by assumption, $\nu_k=\nu$ is admissible.

    Now, we recurse on the second term of the above display 
    using \cref{lem:contraction}:
    \begin{align*}
        \sqrt{\gamma} \, \norm{\sqrt{f_{k-1}} - \sqrt{q^\star}}_{2,\nu_{k-1}}
        \leq \sqrt{\gamma\, C}\norm{\sqrt{f_{k-1}} - \sqrt{\mcT f_{k-2}}}_{2,\mu} + \gamma  \norm{\sqrt{f_{k-2}} - \sqrt{q^\star}}_{2,\nu_{k-2}}
    \end{align*}
    where the inequality uses \cref{lem:contraction}
    while setting $f,f',\nu,\mu$ to $f_{k-1},f_{k-2},\nu_{k-1}$ and $\mu$ (the data generating distribution), respectively,
   	and noting that, by definition, $\nu P\times\pi_{f',q^\star}$ of the \namecref{lem:contraction} is 
    $(\nu_{k-1} P) \times \pi_{f_{k-2},q^\star}=\nu_{k-2}$, and that, as argued before, $\nu_{k-1}$ is admissible.
    
    Continuing this way, and then plugging in back to the first display of the proof, we get
%    Thus, repeated application of \cref{lem:contraction} backwards through time, i.e. $t=[k-1,k-2,...,1]$, gives
    \begin{align*}
        \norm{\sqrt{f_k} - \sqrt{q^\star}}_{2,\nu} 
        &\leq \sqrt{C}\, \sum_{j=1}^k \gamma^{\frac{{k-j}}{2}}\, \norm{\sqrt{f_{j}} - \sqrt{\mcT f_{j-1}}}_{2,\mu} + \gamma^{\frac{k}{2}} \norm{\sqrt{f_0}-\sqrt{q^\star}}_{2,\nu_{0}} \\
        &\leq \sqrt{C}\,\underbrace{ \sum_{j=1}^k \gamma^{\frac{{k-j}}{2}}\, \norm{\sqrt{f_{j}} - \sqrt{\mcT f_{j-1}}}_{2,\mu}}_{S:=} + \gamma^{\frac{k}{2}}\,,
    \end{align*}
    where the second inequality holds because
    by assumption, $f_0$ takes values in $[0,1]$ and so does $q^\star$. Hence,
    $(\sqrt{f_0}-\sqrt{q^\star})^2\le 1$ and thus $\norm{\sqrt{f_0}-\sqrt{q^\star}}_{2,\nu_{0}} \le 1$. 
    
    Now, we bound $S$ defined above: 
    \begin{align*}
        S
        &\leq \max_{\tau \in [1,\dots,k]} \norm{\sqrt{f_{\tau}} - \sqrt{\mcT f_{\tau-1}}}_{2,\mu} \sum_{j=1}^k \gamma^{\frac{k-j}{2}} \\
        &\leq \frac{2}{1-\gamma} \, \max_{\tau \in [1,\dots,k]} \norm{\sqrt{f_{\tau}} - \sqrt{\mcT f_{\tau-1}}}_{2,\mu}\,. \tag{$1/(1-\sqrt{\gamma}) \leq 2/(1-\gamma)$}\\
	\end{align*}
	Chaining the inequalities finishes the proof.
\end{proof}
\begin{remark}
Note that in the proof of the last result it was essential that in the definition of admissibility we allow nonstationary policies.
\end{remark}

With this we are ready to prove \cref{lem:contract-concentration}:
\begin{proof}[Proof of \cref{lem:contract-concentration}]
For the proof let $f_t$ denote the action-value function computed by FQI in step $t=0,1,\dots,k$.
Recall that by construction $f_0\in \mcF$ and that by assumption all functions in $\mcF$ take values in $[0,1]$.
We have
    \begin{align*}
     \norm{ \Delta_{f_k} }_{2,\nu}&=
        \norm{\frac{f_k - q^\star}{\sqrt{f_k + q^\star}}}_{2,\nu} \tag{definition of $\Delta_f$}\\
        &\leq \sqrt{2}\norm{\sqrt{f_k} - \sqrt{q^\star}}_{2,\nu} \tag{first part of \cref{cor:hellinger-triangular-scalar}}\\
        &\leq 
        \sqrt{2}\gamma^{\frac{k}{2}}+
		\frac{2 \sqrt{2}\sqrt{C}}{1-\gamma} \max_{1\le \tau\le k} \norm{ \sqrt{f_\tau}-\sqrt{\mcT f_{\tau-1}}}_{2,\mu}	\tag{\cref{cor:ep}, $0\le f_0\le 1$}\\
		& \le \sqrt{2}\gamma^{\frac{k}{2}}+ \frac{4}{1-\gamma} \, \max_{\tau \in [1,\dots,k]}  
        \norm{ \hell^2( f_{\tau} \Mid \mcT f_{\tau-1}) }_{1,\mu}^{1/2}\,.
%        \left(\int \hell^2 \bigl(f_{\tau}(z) \Mid \mcT f_{\tau-1}\, (z)\bigr) \mu(dz)\right)^{1/2}\,.
		\tag{second part of \cref{cor:hellinger-triangular-scalar}}
	\end{align*}
where in the 
second inequality we used that by assumption $\nu$ is admissible.

For $g: \mcS \times \mcA \to [0,1]$, 
let $\hat{f}_g$ be the function learned by regressing on $g$ via log-loss, i.e., 
%\todoc{perhaps better to introduce $\hat \mcT$
%so that $\hat \mcT g = \hat{f}_g$ (and not introduce $\hat{f}_g$.. ``empirical operator''.. more expressive of what is going on..}
     \[
        \hat{f}_g = \argmin_{f\in\mcF} 
        \sum_{i=1}^n \elllog\left(f(S_i,A_i);\, C_i+\gamma g^\wedge(S_i')\right)\,.
     \]
     Note that $f_{\tau} = \hat{f}_{f_{\tau-1}}$. Hence, 
     \begin{align*}
     \norm{\hell^2( f_{\tau} \Mid \mcT f_{\tau-1}) }_{1,\mu}
	     = \norm{\hell^2( \hat{f}_{f_{\tau-1}} \Mid \mcT f_{\tau-1})}_{1,\mu}
	     \le \max_{g\in \mcF}
	     \norm{
	      \hell^2( \hat{f}_{g} \Mid \mcT g)}_{1,\mu} \tag{because $f_{\tau-1}\in \mcF$}
     \end{align*}
     %\todoc{The last inequality when $\tau=1$ requires that $f_0 \in \mcF$ and at some point we choose $f_0 = 0$. So either
     %need a new assumption, or drop the choice $f_0 = 0$!}
     Since this applies for any $1 \le \tau \le k$, all that remains is to bound the right-hand side of the last display.
     We will use \cref{thm:concentration} for this purpose. This result can be applied because, on the one hand, 
     by \cref{asp:data}, $\EE[ C_i + \gamma g^\wedge(S_i')| S_i,A_i]= \mcT g\, (S_i,A_i)$,
     and by  \cref{asp:completeness}, $\mcT g\in \mcF$ whenever $g\in \mcF$
     and because, again, by \cref{asp:data},
     $(S_i,A_i,C_i,S_{i+1}')$ 
     are independent, identically distributed random variables for $i=1,\dots,n$.
     \todoc[inline]{A  contentious issue is the assumption that $\mcF$ is finite and yet $\mcT \mcF \subset \mcF$.
     Will this \emph{ever} be satisfied? Take the finite case for example. Already in this case, this would require
     that $\mcT^i 0\in \mcF$ for $i=0,1,\dots$ assuming that $0\in \mcF$.. So this will need to be worked on.}
     Thus,   
     \cref{thm:concentration} together with a union bound and recalling that the distribution of $(S_i,A_i)$ is $\mu$ gives 
     that for any $0<\delta<1$,
     \begin{align*}
     \max_{g\in \mcF} \norm{ \hell^2(\hat{f}_g,\mcT g) }_{1,\mu} \le 
     \frac{2\log(|\Fcal|^2 / \delta)}{n}\,.
     \end{align*}
     Putting things together, we get that for any fixed $0<\delta<1$, with probability $1-\delta$,
    \begin{equation*}
         \norm{\Delta_{f_k}}_{2,\nu} \leq \sqrt{\frac{32 C\log\left( \lvert\mcF\rvert^2/\delta\right)}{(1-\gamma)^2 n}} +\sqrt{2} \gamma^{\frac{k}{2}}\,. 	\qedhere
    \end{equation*} 
\end{proof}
\section{Experimental details}\label{sec:appendix-experiment}
In this section, we provide additional experimental details and results. 

In our Atari  experiments, we use the hyperparameters reported in \citet{agarwal2020optimistic}, and we do not tune the hyperparameter for our proposed method. 
The original dataset contains $5$ runs of DQN replay data. Each run of DQN replay data contains $50$ datasets, and each dataset contains $1$ million transitions. 
Due to the memory constraint, we can not load the entire data.
As a result, for each training epoch, we select $5$ datasets randomly, subsample a total of $500$k transitions from the $5$ selected datasets, and perform $100$k updates using the $500$k transitions. 

Figure \ref {fig:atari_all} show the result with clipped losses and unclipped losses. Log-loss consistently outperforms the DQN variants other than C51. 

\begin{figure*}[tbh]
    \centering
    \includegraphics[width=0.48\linewidth]{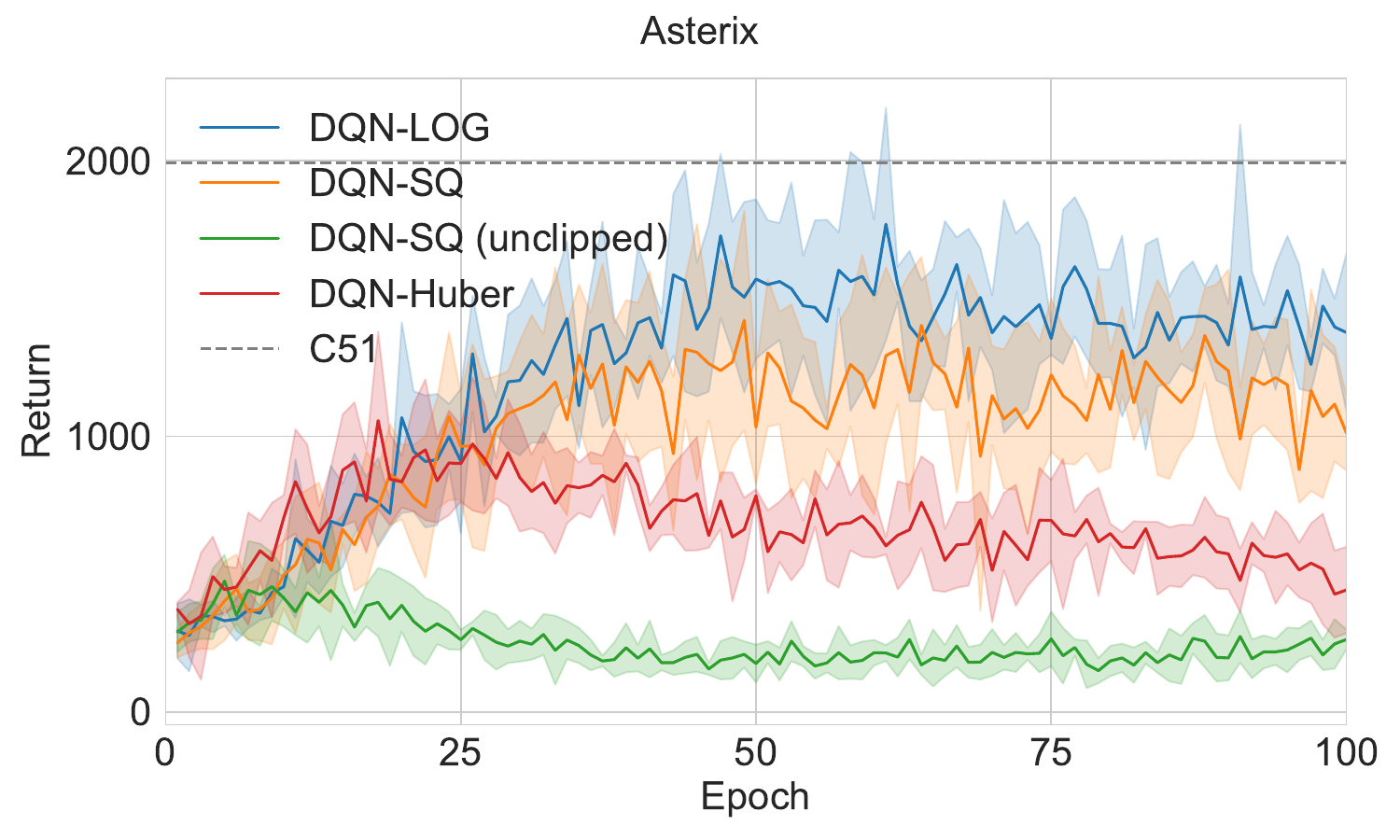}
    \includegraphics[width=0.48\linewidth]{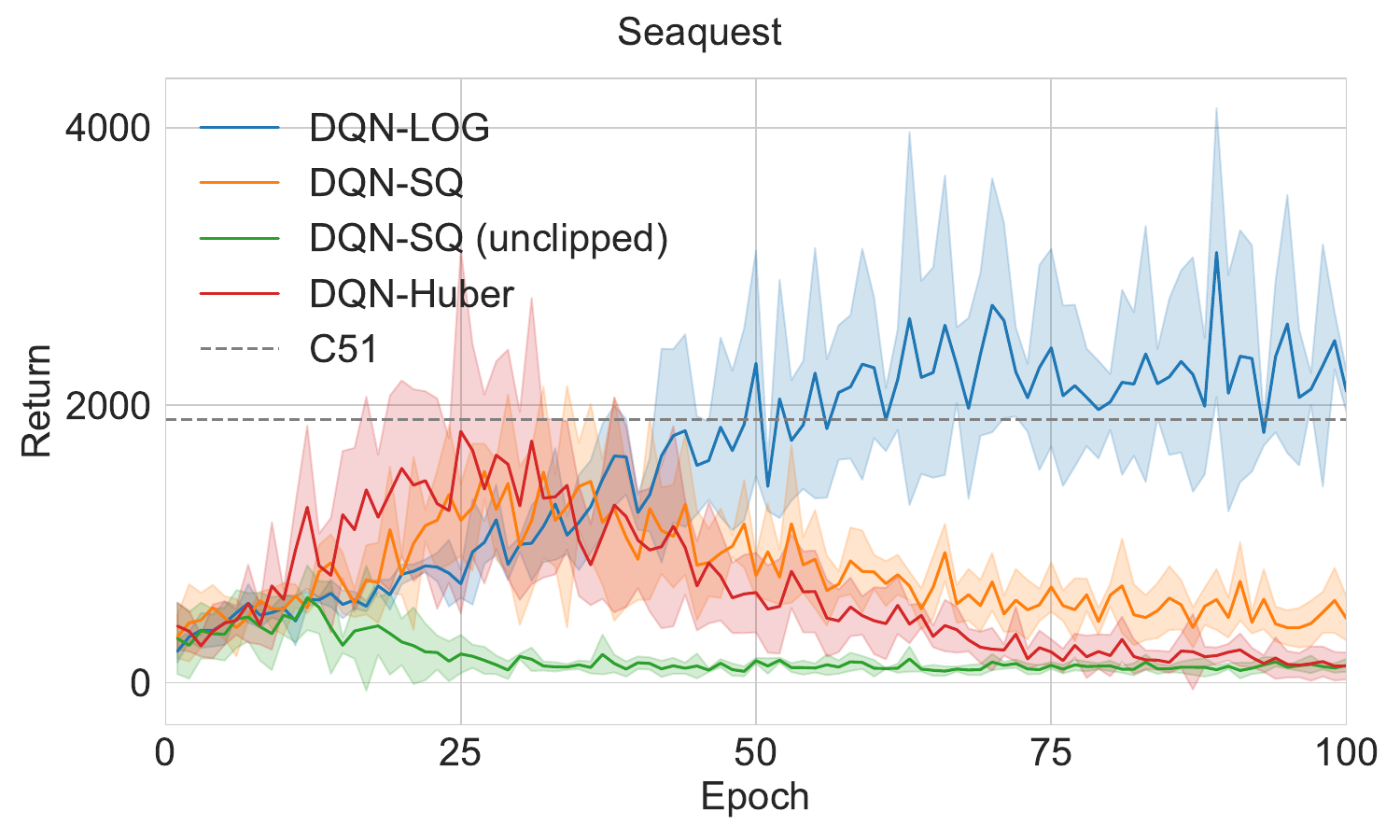}
    \caption{Learning curves on Asterix and Seaquest. The result are averaged over 5 datasets with one standard error. One epoch contains 100k updates.}
    \label{fig:atari_all}
\end{figure*}

\end{document}